\documentclass{article}

\PassOptionsToPackage{numbers,sort&compress}{natbib}
\usepackage[preprint]{neurips_2026}

\usepackage[utf8]{inputenc}
\usepackage[T1]{fontenc}
\usepackage{hyperref}
\usepackage{url}
\usepackage{booktabs}
\usepackage{amsmath}
\usepackage{amssymb}
\usepackage{amsthm}
\usepackage{graphicx}
\usepackage{enumitem}
\usepackage[table]{xcolor}
\usepackage{microtype}
\usepackage{placeins}
\usepackage{capt-of}
\usepackage{needspace}
\usepackage{xfp}
\usepackage{pifont}
\usepackage{enumitem}
\usepackage{mathtools}

\providecommand{\norm}[1]{\left\lVert #1 \right\rVert}

\definecolor{pdePurpleDark}{HTML}{4B2E83}
\definecolor{pdeRowBlue}{HTML}{EEF3F8}
\hypersetup{
  colorlinks=true,
  linkcolor=pdePurpleDark,
  citecolor=pdePurpleDark,
  urlcolor=pdePurpleDark
}
\newcommand{\sci}[2]{\fpeval{round((#1)*10^(#2), 6)}}

\newtheorem{theorem}{Theorem}[section]
\newtheorem{lemma}[theorem]{Lemma}
\newtheorem{proposition}[theorem]{Proposition}

\newtheorem{remark}[theorem]{Remark}

\title{Discovering Physical Directions in Weight Space:\\
Composing Neural PDE Experts}

\author{
  Pengkai Wang\textsuperscript{1,*},
  Pengwei Liu\textsuperscript{2,*},
  Yuanyi Wang\textsuperscript{1,*},
  Guanyu Chen\textsuperscript{2},
  Xingyu Ren\textsuperscript{2},\\
  \textbf{Xiaolong Li}\textsuperscript{2},
  \textbf{Zhongkai Hao}\textsuperscript{3},
  \textbf{Yuting Kong}\textsuperscript{2},
  \textbf{Qixin Zhang}\textsuperscript{4,\textdagger},
  \textbf{Dong Ni}\textsuperscript{2,\textdagger}\\
  \textsuperscript{1}The Hong Kong Polytechnic University,
  \textsuperscript{2}Zhejiang University,
  \textsuperscript{3}Tsinghua University,\\
  \textsuperscript{4}Nanyang Technological University\\
  \textsuperscript{*}Equal contribution.
  \textsuperscript{\textdagger}Corresponding author.\\
}

\begin{document}

\maketitle

\begin{abstract}
Recent advances in neural operators have made partial differential equation (PDE) surrogate modeling increasingly scalable and transferable through large-scale pretraining and in-context adaptation.
However, after a shared operator is fine-tuned to multiple regimes within a continuous physical family, it remains unclear whether the resulting weight-space updates merely form isolated regime experts or reveal reusable physical structure.
Starting from a shared family anchor, we fine-tune low- and high-regime endpoint experts and show that their updates can be separated into a family-shared adaptation and a direction aligned with the underlying physical parameter.
This separation reinterprets endpoint experts as finite-difference probes of a local physical direction in weight space, explaining why static averaging can interpolate between regimes but attenuates endpoint-specific physics.
Building on this perspective, we propose \textbf{Calibration-Conditioned Merge (CCM)}, a post-hoc coordinate readout method for composing neural PDE experts along this physical direction.
Given physical metadata, a calibrated coordinate mapping, or a short observed rollout prefix, CCM infers the target composition coordinate and deploys a single merged checkpoint for the remaining rollout.
This enables training-free regime transfer without ensembling, routing, or target-regime-specific fine-tuning.
We evaluate CCM on the reaction--diffusion system, viscosity-parameterized two-dimensional Navier--Stokes equations, and radial dam-break dynamics.
Across these benchmarks, CCM achieves its strongest gains in extrapolative regimes, reducing out-of-distribution rollout error relative to the family anchor by 54.2\%, 42.8\%, and 13.8\%, respectively.
Further experiments across FNO scales, a DPOT-style backbone, and ablations confirm that endpoint fine-tuning is not arbitrary checkpoint drift, but reveals a calibratable physical direction for training-free transfer across PDE regimes.
\end{abstract}

\section{Introduction}
\label{sec:intro}

Neural operators have emerged as a key methodological framework for constructing efficient surrogate solvers for parametric partial differential equations (PDEs)~\cite{li2021fourier,azizzadenesheli2024neural,hao2023gnot,cao2024laplace,li2020neural,zheng2024alias,liu2025aerogto,liu2025efficient,ren2026foundation,renuncertainty}. 
Whereas early work often targeted narrow equation classes and limited parameter ranges, real scientific problems~\cite{kurth2023fourcastnet,lam2023graphcast,li2021fourier,lu2021learning,kovachki2023neural} require surrogate solvers that remain accurate and stable under changes in material and model coefficients, domain geometry, boundary conditions, spatial and temporal resolution, and overall simulation regime.
To improve such transferability, recent work has advanced PDE foundation models~\cite{hao2024dpot,herde2024poseidon,rahman2024pretraining,mccabe2024mpp,mccabe2025walrus,kochkov2024neuralgcm,bodnar2025aurora} as well as physics-oriented benchmarks and datasets~\cite{takamoto2022pdebench,ohana2024thewell,koehler2024apebench} spanning diverse governing equations, geometries, and physical regimes.
These developments make neural operators increasingly scalable and reusable, but they leave open a basic question about adaptation: after a shared operator is fine-tuned to multiple regimes within the same physical family, do the resulting weight updates merely define isolated regime experts, or do they reveal reusable physical structure?

Adaptation across physical regimes remains difficult.
A single shared base solver may perform well near its training distribution, but its rollout accuracy often deteriorates when key problem parameters, such as viscosity or diffusivity, move into out-of-distribution (OOD) regimes~\cite{pfaff2021meshgraphnets,boussif2022magnet,goswami2022deeptransfer,li2023geofno,li2023gino,helwig2023group,wang2024latent,yin2024scalable,zhao2025dno}. 
Existing adaptation strategies largely follow two paths.
The first learns conditional operators that take regime information as input~\cite{hao2023gnot,wu2024transolver,hao2024dpot,herde2024poseidon,mccabe2024mpp}, but this requires the architecture, training corpus, and conditioning variables to cover the relevant physical state space with sufficient fidelity.
The second fine-tunes a shared base solver into multiple regime-specific experts~\cite{pfaff2021meshgraphnets,boussif2022magnet,goswami2022deeptransfer,li2023geofno,li2023gino,helwig2023group,wang2024latent,yin2024scalable,zhao2025dno}, which can reuse an existing model but produces a fragmented set of checkpoints and offers no principled way to compose them for unseen regimes.
Generic model-merging methods provide useful tools for combining fine-tuned models~\cite{wortsman2022model,matena2022merging,ilharco2023editing,yadav2023ties,yang2024adamerging,yu2024dare,wang2025model,gu2025infifpo,wang2025infigfusion}, yet they typically treat downstream tasks as unordered and select merge coefficients as validation-tuned weights rather than as physically meaningful coordinates.

In this paper, we revisit expert fine-tuning from a physical-coordinate perspective.
In a parametric PDE family, downstream regimes are not arbitrary tasks but ordered points along an underlying physical axis.
If the ideal solver changes smoothly along this axis, then experts fine-tuned from a common family anchor may trace a local trajectory in weight space.
Endpoint experts can therefore be viewed as finite-difference probes of how the shared solver should change from low to high physical regimes.
This leads to our central question: \textit{do same-anchor PDE experts simply represent isolated checkpoints, or do their endpoint--anchor residuals define a reusable direction aligned with the underlying physical parameter?}

\begin{figure}[t]
\centering
\includegraphics[width=0.92\linewidth]{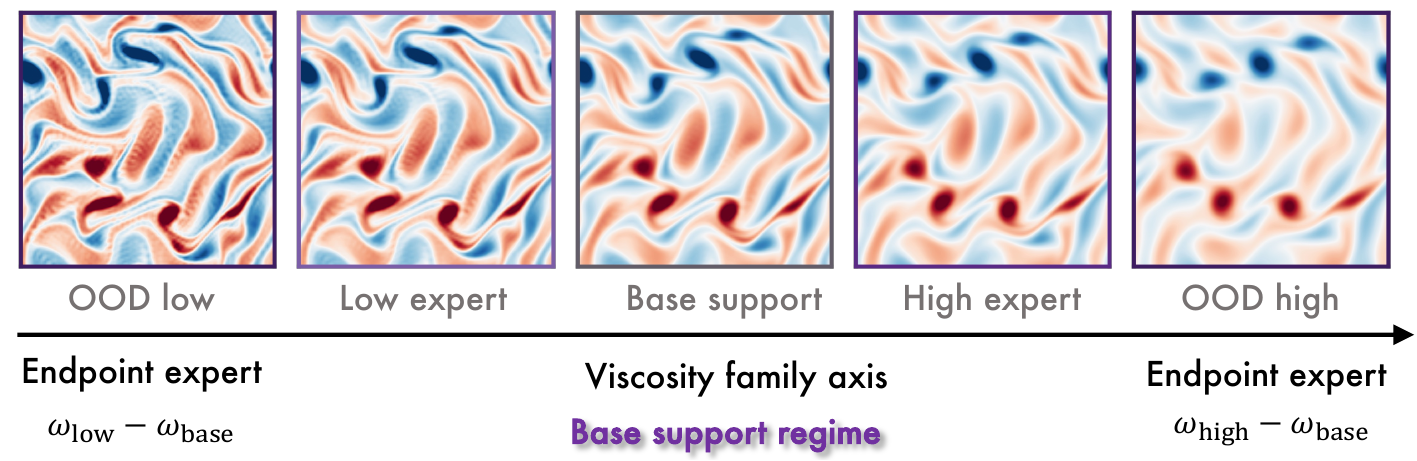}
\caption{\textbf{Endpoint residuals encode physical coordinates.} 
Decomposing endpoint--anchor updates isolates a shared solver adaptation and a signed physical direction. 
This learned direction can align with physical coordinates and support extrapolative rollouts. 
Reversed-ordering controls test whether these gains arise from physical orientation rather than arbitrary interpolation.
}
\label{fig:intro-motivation}
\end{figure}

We answer this question by decomposing low- and high-regime endpoint updates relative to the same anchor.
The mean endpoint residual captures adaptation shared across the PDE family, while the signed residual contrast defines a candidate physical direction in weight space.
This decomposition explains why static averaging can be effective when the shared component dominates, but can also attenuate endpoint-specific physics when the target regime requires a directed shift along the physical axis.
It also turns expert composition into a calibrated coordinate-selection problem: rather than choosing among isolated checkpoints, we can compose a single checkpoint at the coordinate corresponding to the target regime.

Building on this perspective, we propose \textbf{C}alibration-\textbf{C}onditioned \textbf{M}erge (\textbf{CCM}), a post-hoc coordinate readout method for training-free transfer across PDE regimes.
Given either physical metadata, a calibrated coordinate mapping, or a short observed rollout prefix, CCM infers a target composition coordinate along the endpoint-defined direction and deploys one merged checkpoint for the remaining rollout.
The procedure requires no target-regime fine-tuning, dynamic expert routing, or inference-time ensembling.
Thus, CCM is not merely a merging heuristic; it operationalizes the hypothesis that same-anchor endpoint fine-tuning exposes a calibratable physical direction in weight space.

We evaluate CCM on reaction--diffusion dynamics, viscosity-parameterized two-dimensional Navier--Stokes equations, and radial dam-break dynamics.
Across these benchmarks, endpoint residuals consistently separate family-shared solver adaptation from signed physical variation.
CCM achieves its strongest gains in extrapolative regimes, reducing OOD rollout error relative to the family anchor by 54.2\%, 42.8\%, and 13.8\%, respectively.
Further experiments across FNO scales, a DPOT-style backbone, reversed-ordering controls, seed reruns, ablations, and physics diagnostics confirm that the observed gains do not arise from arbitrary checkpoint interpolation.
Instead, endpoint fine-tuning reveals a physically oriented and calibratable weight-space direction for training-free regime transfer.

We summarize the main contributions of this work as follows.
\vspace{-0.25em}
\begin{description}[
    leftmargin=1.35em,
    labelwidth=1.05em,
    labelsep=0.35em,
    itemsep=0.15em,
    topsep=0.25em
]
\item[\textbf{\ding{172}}] \textbf{A physical-coordinate view of PDE expert fine-tuning.}
We formulate same-anchor endpoint experts as finite-difference probes of a local solver trajectory, asking whether endpoint--anchor residuals encode reusable physical structure rather than isolated checkpoint drift.

\item[\textbf{\ding{173}}] \textbf{Calibration-Conditioned Merge.}
We propose CCM, a post-hoc coordinate readout method that composes a single checkpoint from low- and high-regime experts using physical metadata, calibrated coordinates, or a short rollout prefix, without target-regime training, routing, or ensembling.

\item[\textbf{\ding{174}}] \textbf{Evidence for a calibratable physical direction.}
We validate CCM across reaction--diffusion, Navier--Stokes, radial dam break, FNO scales, a DPOT-style backbone, wrong-sign controls, seed reruns, ablations, and physics diagnostics, showing that endpoint fine-tuning reveals a signed physical direction useful for interpolation and extrapolation.
\end{description}
\section{Preliminaries and Problem Setting}
\label{sec:preliminaries}

\noindent\textbf{Parametric PDE family.}
Here we consider time-dependent PDE families on two-dimensional spatial domains
$\Omega\subset\mathbb{R}^{2}$ and a discrete time grid
$0=t_0<t_1<\cdots<t_T$.
The state has $C_u$ physical channels,
$u(t,\cdot):\Omega\to\mathbb{R}^{C_u}$, and the discretized trajectory on an
$H\times W$ grid is
$u_{0:T}\in\mathbb{R}^{(T+1)\times H\times W\times C_u}$.
The initial condition is $x=u_0\in\mathcal{X}$.
Each physical regime is indexed by a scalar family parameter
$\lambda\in\Lambda\subset\mathbb{R}$, such as diffusivity, viscosity, or a
dam-height coordinate in our experiments.
Given $x$, the trajectory solves the compact residual equation
\begin{equation}
    \mathcal{R}_{\lambda}(u;x)=0,
    \qquad
    u(0)=x,
\end{equation}
where $\mathcal{R}_{\lambda}$ collects the PDE operator, coefficients, forcing,
and boundary conditions for regime $\lambda$.
This defines a rollout solution operator
$\mathcal{G}_{\lambda}:\mathcal{X}\to\mathcal{Y}_{T}$,
$\mathcal{G}_{\lambda}(x)=u_{1:T}$, where
$\mathcal{Y}_{T}$ is the future-trajectory space with the same spatial and
channel dimensions.

\noindent\textbf{Rollout evaluation.}
The neural operator with weights $\theta$ induces an autoregressive rollout
operator $\mathcal{F}_{\theta}:\mathcal{X}\to\mathcal{Y}_{T}$, with
$\mathcal{F}_{\theta}(x)=\hat{u}^{\theta}_{1:T}(x)$, whose discretized shape
matches $\mathcal{G}_{\lambda}(x)$.
For an evaluation index set $\mathcal{I}\subseteq\{1,\ldots,T\}$, we measure the mean rollout L2 error
\begin{equation}
    \mathcal{L}_{\lambda,\mathcal{I}}(\theta)
    =
    \mathbb{E}_{x\sim\mathcal{D}_{\lambda}}
    \left[
    \left|
    \mathcal{F}_{\theta}(x)-\mathcal{G}_{\lambda}(x)
    \right|_{\mathcal{I}}
    \right].
\end{equation}
Here $\mathcal{D}_{\lambda}$ is the initial-condition or trajectory
distribution for regime $\lambda$, and $|\cdot|_{\mathcal{I}}$ denotes the
benchmark-normalized rollout L2 over the selected time indices, spatial grid,
and state channels.
The full-rollout L2 used in most tables is
$\mathcal{L}_{\lambda,\mathcal{I}_{\mathrm{full}}}$ with
$\mathcal{I}_{\mathrm{full}}=\{1,\ldots,T\}$.
When a $K$-step calibration prefix is observed, the prefix selector uses
$\mathcal{I}_{\mathrm{cal}}=\{1,\ldots,K\}$ and the reported future-rollout L2
uses $\mathcal{I}_{\mathrm{fut}}=\{K+1,\ldots,T\}$.

\noindent\textbf{Physical-axis extrapolation.}
For each family, we fix two endpoint regimes
$\lambda_{\mathrm{low}}<\lambda_{\mathrm{high}}$ before any target-regime
evaluation.
These two regimes define both the coordinate scale and the endpoint experts
introduced below; interpolation and OOD evaluation targets are not used to
choose them.
We normalize the physical coordinate by
\begin{equation}
    s(\lambda)=c(\lambda)
    =
    \frac{2(\lambda-\lambda_{\mathrm{mid}})}
    {\lambda_{\mathrm{high}}-\lambda_{\mathrm{low}}},
    \qquad
    \lambda_{\mathrm{mid}}
    =
    \frac{\lambda_{\mathrm{low}}+\lambda_{\mathrm{high}}}{2}.
\end{equation}
With this normalization, the endpoint regimes, and hence their experts, are at
$s=\pm1$, interpolation has $|s|\le1$, and extrapolation $|s|>1$.
We use OOD only for extrapolation along this physical axis, not arbitrary distribution shift.
An OOD-high target lies beyond the high-parameter expert.
This ties the merge coordinate to an external physical reference, independent of learned weights or validation loss.

\noindent\textbf{Controlled same-base setting.}
A base neural operator with weights $\theta_0$ is trained on support regimes $\Lambda_{\mathrm{sup}}\subset\Lambda$.
From this base, the endpoint experts $\theta_{\mathrm{low}}$ and
$\theta_{\mathrm{high}}$ are obtained by fine-tuning only on data from
$\lambda_{\mathrm{low}}$ and $\lambda_{\mathrm{high}}$, respectively.
They are fixed once by the family construction rather than selected by target
validation loss or by a merge sweep.
At test time, these weights are composed into a single checkpoint.
Because all checkpoints share the same architecture, normalization, and rollout protocol, their weight displacements are directly comparable.
This shared lineage defines the composition path in Sec.~\ref{sec:method}.

\noindent\textbf{Allowed test-time information.}
We evaluate three controlled PDE families: reaction--diffusion dynamics, 2D incompressible Navier--Stokes flow, and radial dam-break transients.
At test time, checkpoint selection may use physical metadata or a short observed rollout prefix, but no gradient updates are allowed on the target regime.
When using a prefix of frames $K$, $\mathcal{I}_{\mathrm{cal}}$ is used only
for calibration and the future-rollout reported L2 is calculated in
$\mathcal{I}_{\mathrm{fut}}$.

\section{Calibration-Conditioned Merge}
\label{sec:method}

Calibration-Conditioned Merge implements the mechanism proposed in the introduction: same-anchor endpoint experts can reveal a local physical coordinate in weight space.
The method starts from the fixed objects in Sec.~\ref{sec:preliminaries}: a family anchor $\theta_0$ and two endpoint experts $\theta_{\mathrm{low}}$ and $\theta_{\mathrm{high}}$.
The endpoints are chosen by the physical family construction before target evaluation.
Concretely, $\theta_{\mathrm{low}}$ is fine-tuned from $\theta_0$ on the low endpoint regime $\lambda_{\mathrm{low}}$, and $\theta_{\mathrm{high}}$ is fine-tuned from $\theta_0$ on the high endpoint regime $\lambda_{\mathrm{high}}$.
The two endpoints use the same architecture, normalization, rollout metric, and fine-tuning protocol, and they are not chosen by target-regime validation or by an $\alpha$ sweep.
This separates endpoint selection from coordinate selection: CCM only selects
$\hat{\alpha}$ on the fixed endpoint-defined line.
Because the endpoint regimes are placed on an external physical axis with normalized coordinates $s=-1$ and $s=+1$, their residual contrast can be treated as a signed low-to-high direction rather than an arbitrary checkpoint difference.
Our local working hypothesis is that, near the family anchor, same-lineage endpoint fine-tuning defines a first-order physical coordinate,
\begin{equation}
    \theta^{\star}(s)
    \approx
    \theta_0+\Delta^{+}+s\Delta^{-},
    \qquad |s|\ \text{near the endpoint interval},
    \label{eq:local-family-coordinate}
\end{equation}
where $\Delta^{+}$ denotes the adaptation shared by endpoint experts, and $\Delta^{-}$ denotes signed low-to-high physical variation.
CCM builds two components from endpoint experts, then selects a coordinate for the target regime.
App.~\ref{sec:app:theory} provides a conditional local secant interpretation of this construction.

\begin{figure}[t]
\centering
\includegraphics[width=1.0\linewidth]{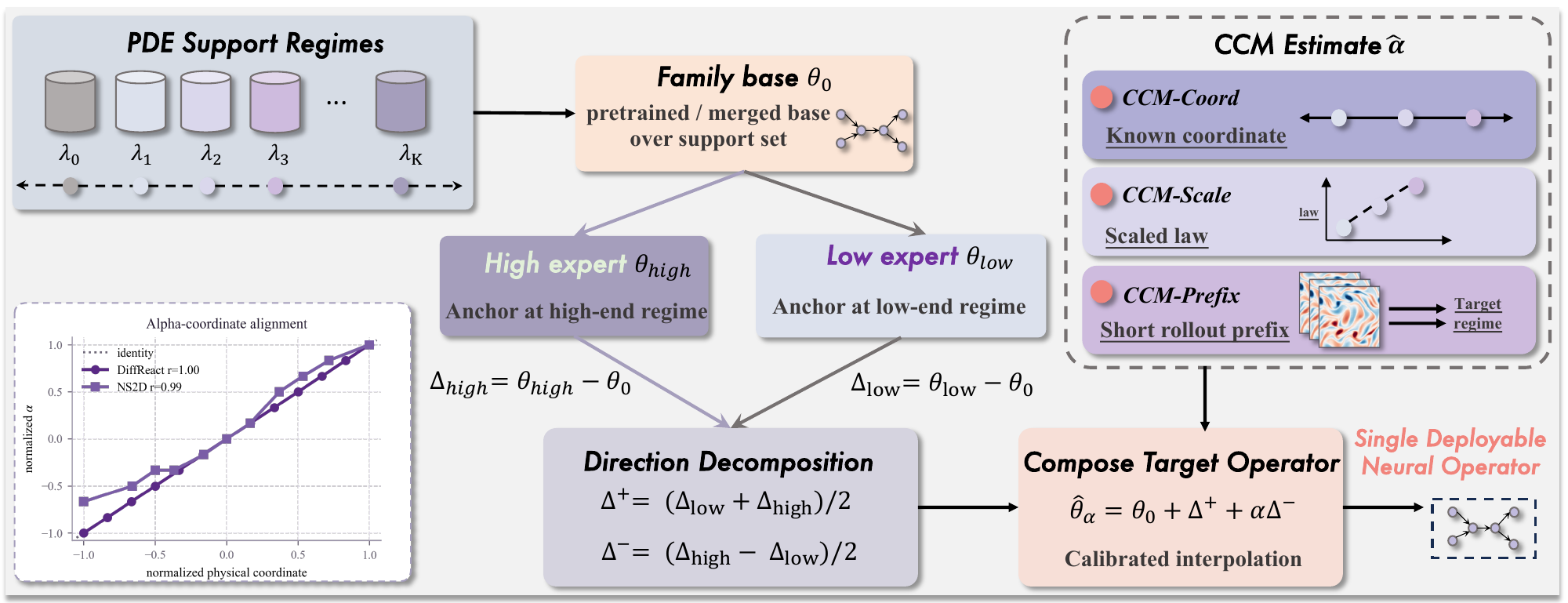}
\caption{\textbf{Calibration-Conditioned Merge framework.}
A shared family anchor $\theta_0$ is trained on support regimes and fine-tuned
into low/high endpoint experts.
Their endpoint--anchor residuals define a shared adaptation $\Delta^{+}$ and a
signed physical direction $\Delta^{-}$.
CCM selects a target coordinate $\hat{\alpha}$ from metadata, scale calibration,
or a short rollout prefix, and instantiates one checkpoint
$\hat{\theta}_{\lambda}=\theta_0+\Delta^{+}+\hat{\alpha}\Delta^{-}$ for
rollout.}
\label{fig:method-framework}
\end{figure}

\subsection{Endpoint residuals define a coordinate family}
\label{sec:method:decomposition}

For endpoint experts fine-tuned from the same anchor, define the endpoint--anchor residuals
\begin{equation}
    \Delta_{\mathrm{low}}=\theta_{\mathrm{low}}-\theta_0,
    \qquad
    \Delta_{\mathrm{high}}=\theta_{\mathrm{high}}-\theta_0 .
\end{equation}
Because both residuals are measured from the same anchor and the endpoints are ordered along the physical axis, we decompose them into
\begin{equation}
    \Delta^{+}
    =
    \frac{\Delta_{\mathrm{low}}+\Delta_{\mathrm{high}}}{2},
    \qquad
    \Delta^{-}
    =
    \frac{\Delta_{\mathrm{high}}-\Delta_{\mathrm{low}}}{2}.
    \label{eq:shared-signed}
\end{equation}
The shared component $\Delta^{+}$ captures what both endpoint experts agree to
change from the anchor.
The signed component $\Delta^{-}$ captures the low-to-high variation oriented along the physical axis.
For any composition coordinate $\alpha$, CCM defines the checkpoint family
\begin{equation}
    \theta(\alpha)
    =
    \theta_0+\Delta^{+}+\alpha\Delta^{-}
    =
    \frac{1-\alpha}{2}\theta_{\mathrm{low}}
    +
    \frac{1+\alpha}{2}\theta_{\mathrm{high}} .
    \label{eq:ccm-family}
\end{equation}
Thus, $\alpha=-1$ and $\alpha=+1$ recover the low and high endpoint experts, respectively.
The value $\alpha=0$ gives the arithmetic average of the two endpoint experts, while $|\alpha|>1$
extrapolates beyond the endpoint interval along the same signed direction.

This parameterization separates two effects that static averaging conflates.
If endpoint fine-tuning mainly learns a family-shared solver correction, then $\Delta^{+}$ should dominate and $\alpha=0$ should be competitive.
If target regimes require endpoint-specific physics, then $\Delta^{-}$ should matter and nonzero $\alpha$ should improve rollout behavior.
The extrapolation gains at $|\alpha|>1$ are especially informative because they
test whether the signed residual carries useful physical variation beyond the trained endpoints.

\subsection{Selecting the composition coordinate}
\label{sec:method:ccm}

The family $\theta(\alpha)$ defines a one-dimensional coordinate line in weight
space.
CCM asks whether a useful coordinate can be selected from information available
at test time:
\begin{equation}
    \hat{\alpha}
    =
    \pi_{\mathrm{CCM}}(\mathcal{I}_{\lambda}),
    \qquad
    \hat{\theta}_{\lambda}
    =
    \theta(\hat{\alpha}),
    \label{eq:ccm-selector}
\end{equation}
where $\mathcal{I}_{\lambda}$ denotes the allowed calibration information for
the target regime.
After selection, rollout uses one checkpoint and one forward path.
The clipping bounds $\alpha_{\min}$ and $\alpha_{\max}$ are fixed before test
evaluation by the deployment range or by the finite coordinate bank used in the
corresponding experiment.

\textbf{CCM-Coord.}
This selector directly tests whether the signed residual follows the measured physical axis.
When the normalized physical coordinate is known and calibrated to the endpoint scale, CCM uses
\begin{equation}
    \hat{\alpha}_{\mathrm{coord}}(\lambda)
    = \mathrm{clip}\!\left(s(\lambda),\alpha_{\min},\alpha_{\max}\right).
    \label{eq:ccm-coord}
\end{equation}
Here the endpoint construction fixes
$s(\lambda_{\mathrm{low}})=-1$ and $s(\lambda_{\mathrm{high}})=+1$; targets
inside this interval use interpolation coordinates, while OOD targets use
extrapolation coordinates before clipping.

\textbf{CCM-Scale.}
The physical coordinate may have the correct ordering but a mismatched
weight-space scale.
CCM therefore uses a globally scaled coordinate,
\begin{equation}
    \hat{\alpha}_{\mathrm{scale}}(\lambda)
    =
    \mathrm{clip}\!\left(\gamma s(\lambda),\alpha_{\min},\alpha_{\max}\right).
    \label{eq:ccm-scale}
\end{equation}
The scalar $\gamma$ is selected once on held-out validation regimes and then
fixed for all inference regimes.

\textbf{CCM-Prefix.}
When metadata are limited or missing, CCM chooses a coordinate from a finite bank
$\mathcal{A}$ based on a short observed rollout prefix:
\begin{equation}
    \hat{\alpha}_{\mathrm{prefix}}(x)
    =
    \arg\min_{\alpha\in\mathcal{A}}
    \widehat{\mathcal{L}}_{x,\mathcal{I}_{\mathrm{cal}}}
    \!\left(\theta(\alpha)\right),
    \qquad
    \mathcal{I}_{\mathrm{cal}}=\{1,\ldots,K\},
    \label{eq:ccm-prefix}
\end{equation}
where $\widehat{\mathcal{L}}_{x,\mathcal{I}_{\mathrm{cal}}}$ denotes the
single-trajectory prefix loss defined by the rollout metric in Sec.~\ref{sec:preliminaries}.
The chosen checkpoint is subsequently assessed on $\mathcal{I}_{\mathrm{fut}}=\{K+1,\ldots,T\}$, excluding the calibration frames from the reported future-rollout error.
Hence, CCM-Prefix uses the short prefix solely to select $\alpha$; it does not perform any weight updates, and the protocol details, including coordinate banks and prefix lengths, are given in App.~\ref{sec:app:protocols}.

\paragraph{Diagnostic oracle coordinate.}
For analysis only, we also report
\begin{equation}
    \alpha^{\star}_{\mathcal{A}}(\lambda)
    =
    \arg\min_{\alpha\in\mathcal{A}}
    \mathcal{L}_{\lambda,\mathcal{I}}\!\left(\theta(\alpha)\right),
    \label{eq:oracle-alpha}
\end{equation}
where $\mathcal{I}$ is the evaluation window for the corresponding task.
It measures whether the endpoint-defined line contains useful target-regime solvers and separates the capacity of the line from the difficulty of selecting its coordinate.
Fig.~\ref{fig:coordinate-loss-strip} compares this diagnostic coordinate with the independently specified physical coordinate $s(\lambda)$.

\section{Experiments}
\label{sec:exp}

In this section, we evaluate whether the endpoint-residual coordinate defined in Sec.~\ref{sec:method} is both physically readable and useful for rollout. 
The experiments follow a mechanism-driven sequence:
\textbf{(i)} whether endpoint residuals separate shared solver adaptation from signed physical variation;
\textbf{(ii)} whether the signed coordinate tracks an independently specified physical axis;
\textbf{(iii)} whether deployable CCM selectors improve beyond-endpoint rollout; and
\textbf{(iv)} whether the behavior persists across controls, architectures, and physics-aware diagnostics.

\textbf{Experimental setting.}
We evaluate our framework on three PDE families with distinct dynamics: reaction--diffusion equations~\cite{takamoto2022pdebench}, two-dimensional Navier--Stokes flow parameterized by viscosity~\cite{li2021fourier}, and radial dam-break equations (RDB)~\cite{takamoto2022pdebench}. 
For each family, we split a scalar physical axis into support, endpoint, and evaluation regimes. 
A family base is trained exclusively on the support regimes, from which low- and high-regime experts are subsequently fine-tuned. 
All model compositions are then executed post hoc and evaluated by the autoregressive rollout metric defined in Sec.~\ref{sec:preliminaries}.
To preclude information leakage in the RDB CCM-Prefix evaluation, the $K=4$ observed prefix is used only for coefficient selection; the future-rollout L2 is computed on the remaining horizon.
More detailed hyperparameters and data construction are recorded in App.~\ref{sec:app:protocols}, App.~\ref{sec:app:data}, and App.~\ref{sec:app:case-studies}.

\textbf{Backbones and baselines.}
Our primary architecture is FNO~\cite{li2021fourier,lu2021learning,kovachki2023neural}.
We also use a DPOT-style backbone~\cite{hao2024dpot} to test whether the residual geometry we observe depends on the operator architecture.
Base denotes the shared family model, while Expert-low and Expert-high are the two endpoint fine-tuned models.
The endpoint average is the checkpoint obtained by averaging the endpoints at $\alpha=0$.
The best fixed $\alpha$ is a single mixing coefficient shared across all targets in a family.
Oracle $\alpha$ is tuned per target and used only diagnostically to assess the expressivity of the endpoint residual line.
CCM-Coord, CCM-Scale, and CCM-Prefix denote plain coordinate conditioning, scale-normalized coordinate conditioning, and selection from a finite bank based on a rollout prefix, respectively.
We also report generic post-hoc merging baselines: Task Arithmetic, TIES, and DARE~\cite{wortsman2022model,matena2022merging,ilharco2023editing,yadav2023ties,yu2024dare}.
Conditional operators and target-specific retraining tackle a complementary adaptation setting~\cite{hao2023gnot,wu2024transolver,herde2024poseidon,mccabe2024mpp}: they evaluate newly trained target-conditioned solvers, not the structure encoded in endpoint updates from a shared base model.

\subsection{Endpoint residuals separate shared adaptation from signed physical variation}
\label{sec:exp:diff-react}
\label{sec:exp:finding-shared-directional}

We first use controlled DiffReact axes as a mechanism test.
Each experiment varies one physical parameter while keeping the PDE family and the model protocol fixed.
The $k$ and $D_v$ axes test whether endpoint averaging captures the solution adaptation shared by both experts.
The harder $D_u$ medium-gap axis tests whether transfer requires signed low-to-high residual contrast.

\begin{center}
\centering
\includegraphics[width=0.49\linewidth]{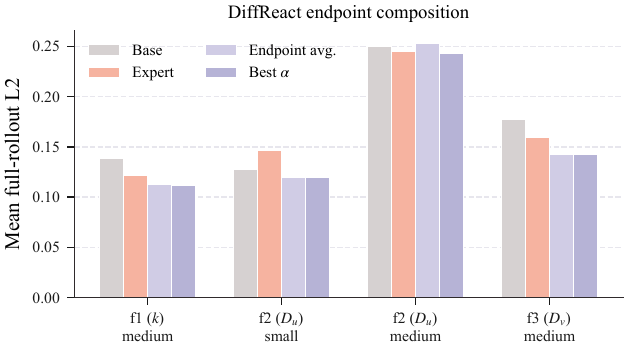}
\hfill
\includegraphics[width=0.49\linewidth]{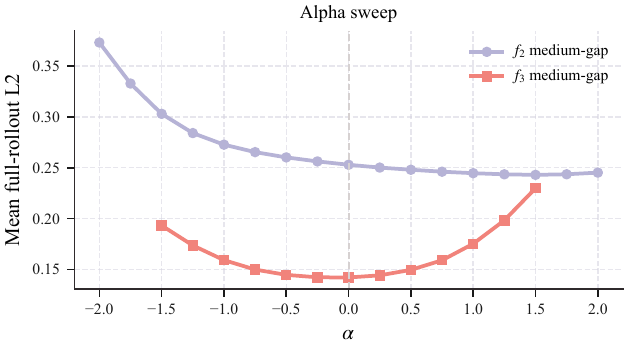}
\captionof{figure}{\textbf{Shared and signed endpoint structure on controlled DiffReact axes.}
Left: endpoint averaging is effective on merge-friendly axes, indicating a
reusable shared component $\Delta^{+}$.
Right: the $f_3$ curve is minimized at $\alpha=0$, whereas the harder
$f_2$ medium-gap setting prefers a nonzero signed family direction
$\Delta^{-}$.}
\label{fig:diff-react-mechanism}
\end{center}

Fig.~\ref{fig:diff-react-mechanism} shows how the same endpoint construction can produce two regimes.
On the $k$ and $D_v$ axes, endpoint averaging improves over the base and is close to, or identical with, the best static point on the $\alpha$ sweep.
For example, on the $D_v$ medium-gap axis, the error decreases from
$\sci{1.775}{-1}$ to $\sci{1.422}{-1}$, and the sweep is minimized at
$\alpha=0$.
This is the shared-adaptation regime: both endpoint experts learn a similar family-wide correction to the anchor.

The $D_u$ medium-gap setting behaves differently.
Endpoint averaging yields a slightly higher error than the baseline ($\sci{2.529}{-1}$ versus $\sci{2.505}{-1}$), whereas introducing a nonzero coordinate with $\alpha = 1.50$ decreases the error to $\sci{2.430}{-1}$.
Although the absolute improvement is smaller than that observed in the merge-friendly axes, the sign of the effect is crucial: averaging suppresses information that is beneficial for the target regime, while traversing the signed direction partially restores this information.
Consequently, endpoint averaging and coordinate-conditioned composition interrogate different components of the endpoint residuals, rather than constituting mutually exclusive merging heuristics.
This distinction provides the first empirical evidence in favor of the shared/signed decomposition.



\Needspace{8\baselineskip}
\subsection{The signed coefficient tracks an independent physical axis}
\label{sec:exp:finding-alpha-coordinate}

We next examine whether the signed residual direction is accessible to physical interpretation. 
For each target regime, we compare two independently obtained quantities: \textbf{(i)} the diagnostic oracle coordinate selected from the coordinate bank via rollout loss, and \textbf{(ii)} the normalized physical coordinate provided by the data generator or defined by the benchmark construction. 
If $\Delta^{-}$ aligns with the physical family axis, these two quantities are expected to covary.

\begin{center}
\centering
\includegraphics[width=\linewidth]{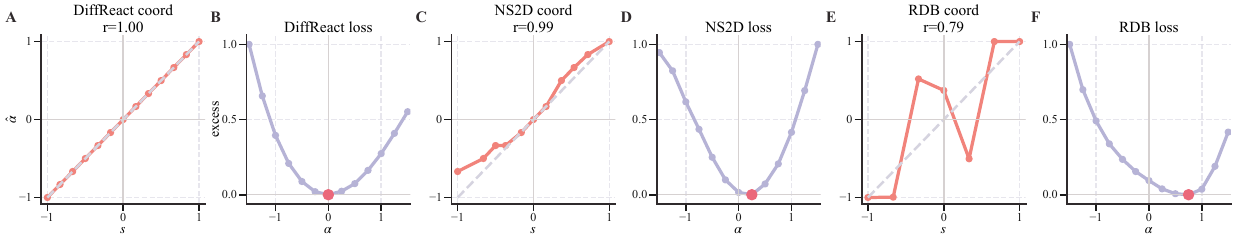}
\captionof{figure}{\textbf{Coordinate law and endpoint-line smoothness across PDE families.}
Panels A, C, and E present a comparison between the normalized physical coordinate and the corresponding diagnostic-oracle-derived coordinate.
Panels B, D, and F display the normalized excess loss evaluated along the same endpoint-defined coordinate axis.
DiffReact and NS2D show strong coordinate alignment, whereas RDB exhibits weaker alignment with the scalar metadata. Nonetheless, RDB produces a smoothly varying, well-parameterized one-dimensional selection path.
}
\label{fig:coordinate-loss-strip}
\end{center}

Panels A, C, and E of Fig.~\ref{fig:coordinate-loss-strip} show that the diagnostic coordinate closely tracks the physical axis in both DiffReact and NS2D.
For DiffReact, the alignment is effectively perfect at the shown precision ($r = 1.00$), and NS2D also shows a very strong correspondence ($r = 0.99$).
These high correlations are the main coordinate-readout evidence: the rollout-selected coefficient is not an arbitrary validation weight but covaries systematically with an independently defined physical coordinate.
By contrast, RDB shows a weaker relationship ($r = 0.79$), indicating that scalar metadata alone do not fully parameterize the free-surface transient and that a dynamic calibration signal is required.

Panels B, D, and F show smooth loss profiles along the endpoint-defined coordinate axis. This smoothness is essential because CCM-Prefix searches a finite coordinate bank; if the loss landscape were irregular or misaligned with the task, prefix selection would be unstable. 
Instead, the profiles reveal a structured valley, with the loss-minimizing coordinate shifting systematically as the target regime changes.
Thus, the weaker RDB coordinate relationship does not undermine the endpoint-defined axis; it instead motivates prefix-based coordinate selection when static metadata are insufficient.

Coordinate choice also affects predictive performance.
Along the dense DiffReact axis, CCM-Coord reduces mean full-rollout L2 from $\sci{1.316}{-1}$ to $\sci{8.24}{-2}$ (a $37.4\%$ decrease).
On the OOD slice, averaging endpoints barely helps, changing error from $\sci{1.88}{-1}$ to $\sci{1.83}{-1}$, whereas CCM-Coord attains $\sci{8.6}{-2}$.
Thus, the coordinate is most useful when the target lies away from the center and the sign of the physical variation matters.
A coordinate law with the wrong sign provides a negative control: it keeps the same checkpoints and magnitudes but flips orientation, increasing error by factors of $2.61$ overall and $3.91$ on OOD tasks.
This rules out a naive ``any interpolation helps'' view: performance depends on choosing the physically correct direction, not merely moving along the checkpoint line.

\subsection{Coordinate selection improves beyond-endpoint rollouts}
\label{sec:exp:finding-ood}
\label{sec:exp:finding-prefix}

We next assess extrapolation along the family axis, where target regimes lie outside the interval spanned by the training endpoints.
This is the setting in which static averaging is expected to attenuate endpoint-specific phenomena, since it effectively removes the signed residual direction.
It also probes whether the endpoint-defined coordinate remains informative beyond the regimes for which the endpoint experts were trained.
We consider three distinct information configurations.
DiffReact dense represents a calibrated-coordinate scenario: the underlying physical coordinate is known, and CCM-Coord directly uses this coordinate.
NS2D viscosity corresponds to an ordered-but-misscaled scenario: the coordinate is aligned with the residual direction, but a single global rescaling improves expert selection, yielding the CCM-Scale variant.
RDB constitutes a weaker-metadata scenario: the available scalar metadata are comparatively unreliable, so CCM-Prefix instead uses a short observed prefix of the trajectory to infer an appropriate coordinate before evaluating future rollouts.

\par\smallskip
\noindent\begin{minipage}{\linewidth}
\centering\small
\captionof{table}{\textbf{OOD family-axis validation.}
DiffReact and NS2D report full-rollout L2 on extrapolative regimes.
RDB reports future-rollout L2 after the $K=4$ prefix used only for
CCM-Prefix selection.
Coordinate- or prefix-conditioned selection gives the main improvement.}
\label{tab:ood-family-axis}
\begin{tabular}{llrrr}
\toprule
\textbf{Benchmark} & \textbf{Method} & \textbf{OOD mean} & \textbf{OOD worst} & \textbf{Gain vs base} \\
\midrule
DiffReact f2 dense & Base & \sci{1.88}{-1} & \sci{2.28}{-1} & -- \\
 & Endpoint average & \sci{1.83}{-1} & \sci{2.24}{-1} & 2.4\% \\
 & Best fixed & \sci{1.83}{-1} & \sci{2.24}{-1} & 2.4\% \\
\rowcolor{pdeRowBlue} & CCM-Coord & \textbf{\sci{8.6}{-2}} & \textbf{\sci{1.00}{-1}} & \textbf{54.2\%} \\
 & Oracle per-task & \sci{8.6}{-2} & \sci{1.00}{-1} & 54.2\% \\
\midrule
NS2D viscosity & Base & \sci{5.91}{-1} & \sci{6.78}{-1} & -- \\
 & Endpoint average & \sci{5.86}{-1} & \sci{7.17}{-1} & 0.9\% \\
 & Best fixed & \sci{5.90}{-1} & \sci{6.70}{-1} & 0.2\% \\
\rowcolor{pdeRowBlue} & CCM-Scale & \textbf{\sci{3.38}{-1}} & \textbf{\sci{3.47}{-1}} & \textbf{42.8\%} \\
 & Oracle alpha & \sci{3.27}{-1} & \sci{3.29}{-1} & 44.7\% \\
\midrule
RDB high-center & Base & \sci{4.71}{-1} & \sci{5.34}{-1} & -- \\
 & Endpoint average & \sci{4.50}{-1} & \sci{5.09}{-1} & 4.5\% \\
 & Best fixed & \sci{4.42}{-1} & \sci{5.23}{-1} & 6.1\% \\
\rowcolor{pdeRowBlue} & CCM-Prefix & \textbf{\sci{4.06}{-1}} & \textbf{\sci{4.83}{-1}} & \textbf{13.8\%} \\
\bottomrule
\end{tabular}

\end{minipage}
\par\smallskip

Tab.~\ref{tab:ood-family-axis} summarizes the OOD comparison.
On DiffReact, endpoint averaging and the best fixed $\alpha$ yield only a $2.4\%$ relative OOD improvement, whereas CCM-Coord reduces the OOD mean from $0.188$ to $0.086$.
The close agreement between CCM-Coord and the oracle coordinate shows that, in this dense-axis regime, the physical metadata suffice to localize the relevant point along the endpoint-defined interpolation line.

On NS2D, the qualitative OOD behavior is similar but harder: endpoint averaging barely changes the OOD mean ($0.591 \to 0.586$), and the best fixed $\alpha$ is essentially unchanged ($0.590$).
In contrast, CCM-Scale achieves an OOD mean of $0.338$, a $42.8\%$ OOD error reduction relative to the baseline.
This indicates that the signed direction is informative, but its coordinate scale must be calibrated for viscosity-parameterized rollouts.
For RDB, endpoint averaging and the best fixed $\alpha$ give OOD mean future L2 errors of $0.450$ and $0.442$, respectively, compared with $0.406$ for CCM-Prefix.
Although smaller than in DiffReact or NS2D, this gain remains meaningful because selection uses only a short prefix and the metric omits the prefix frames, indicating that early rollout dynamics encode both the correct sign and magnitude of the endpoint direction.

Across all three model families, static averaging gives only minor OOD gains, while coordinate- or prefix-conditioned selection yields most of the improvement. 
This aligns with the proposed mechanism: shared adaptation aids inputs near the training distribution center, but endpoint-specific signed parameter variation is needed when targets fall outside the range spanned by the training set.

\subsection{Controls, robustness, and auxiliary diagnostics}
\label{sec:exp:finding-scale-backbone}

We finally test whether the observed coordinate structure persists beyond the main rollout metric and the default FNO backbone.
Fig.~\ref{fig:cross-domain-evidence-atlas} summarizes the key robustness results.
The wrong-sign control consistently increases error, confirming that improvements depend on the endpoint's physical orientation, not arbitrary interpolation.

\begin{center}
\centering
\includegraphics[width=\linewidth]{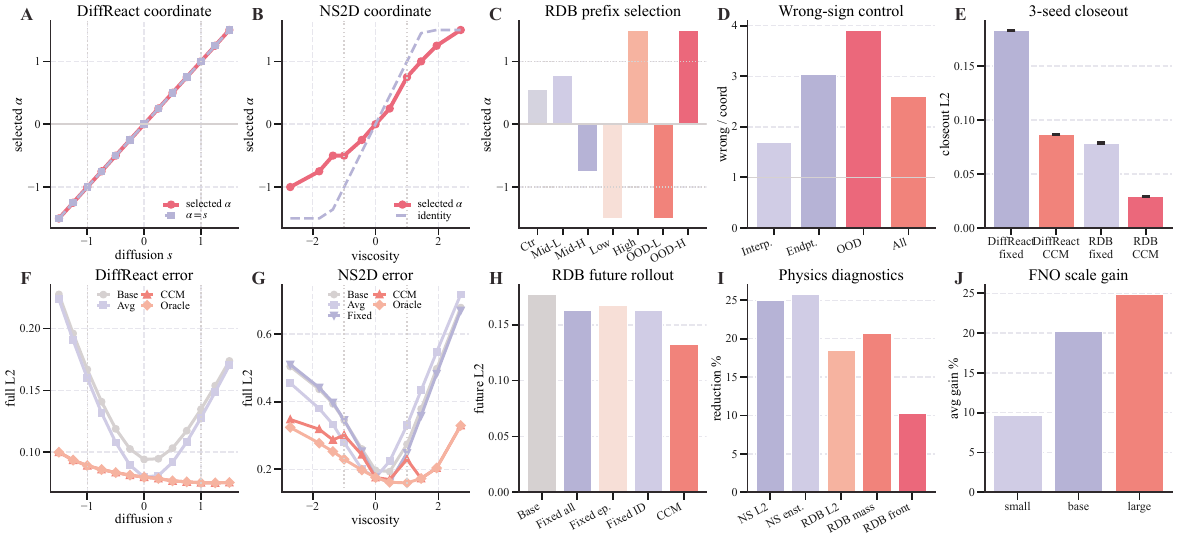}
\captionof{figure}{\textbf{Cross-domain evidence atlas.}
A--B show coordinate alignment; C, RDB prefix selection; D, wrong-sign coordinate control; E, matched-seed independent reruns; F--H, rollout error comparisons; I, auxiliary diagnostic reductions; J, the FNO scale trend.
}
\label{fig:cross-domain-evidence-atlas}
\end{center}

Matched-seed reruns retain the qualitative advantages of both main strategies: metadata-based coordinate selection on DiffReact and prefix-based selection on RDB.
FNO scaling experiments show endpoint averaging consistently outperforming the anchor for small, base, and large models, with DiffReact gains increasing from $9.6\%$ to $24.8\%$.
A compact DPOT-style backbone yields the same ordering: static averaging outperforms the base model, and coordinate-conditioned composition achieves the lowest error.

These tests suggest that the residual geometry is not specific to one FNO scale, operator architecture, or random seed.
We also compute auxiliary physics diagnostics in App.~\ref{sec:app:physics-compute}. 
These diagnostics mirror the rollout metrics: coordinate- or prefix-conditioned selection improves physically meaningful summaries in NS2D and RDB, without altering the single-checkpoint deployment protocol. 
We place full definitions and task-specific metrics in the appendix to avoid overloading the main text with benchmark-specific diagnostics.

Overall, the evidence supports local physical-coordinate geometry within same-anchor, same-family settings. 
Endpoint composition separates shared adaptation from signed physical variation, the signed coordinate follows an independently specified physical axis when metadata are reliable, and coordinate- or prefix-conditioned selection improves extrapolative rollout across three PDE families.

\section{Conclusion}
\label{sec:conclusion}

This work shows that same-anchor neural PDE experts form a continuous family, not isolated checkpoints. 
Decomposing endpoint--anchor residuals separates shared solver adaptation from signed physical variation along the PDE axis, which explains why static averaging suffices when shared adaptation dominates, but fails when extrapolation demands endpoint-specific physics. 
Exploiting this structure, CCM selects a composition coordinate from physical metadata, a calibrated scale, or a short rollout prefix, yielding a single checkpoint without test-time fine-tuning, routing, or ensembling. 
Across diverse PDE families, CCM improves most on family-axis extrapolation, where naive averaging suppresses signed physical information. 
Controlled expert composition thus serves as both a training-free adaptation mechanism and a diagnostic: it reveals how neural solvers encode continuous physical variation as interpretable weight-space directions.

\vspace{-1mm}
\paragraph{Limitations.}
\label{sec:limitations}
CCM presupposes a common local geometry, compatible architectures, and stable fine-tuning dynamics, but does not provide a universal OOD guarantee.  
CCM-Prefix additionally depends on a brief observation horizon, which precludes true zero-shot deployment.  
The coordinate assumption breaks down when endpoint gaps grow large, when we move far from the anchor's local chart, or when the checkpoint manifold exhibits high curvature.  
Our analysis currently focuses on single scalar directions; extending to multi-parameter PDE families will likely require cross-axis experts or higher-order charts to represent coupled interactions.

\clearpage
\phantomsection
\label{page:after-main}
\bibliographystyle{unsrtnat}
\bibliography{references}

@article{li2021fourier,
  title={Fourier neural operator for parametric partial differential equations},
  author={Li, Zongyi and Kovachki, Nikola and Azizzadenesheli, Kamyar and Liu, Burigede and Bhattacharya, Kaushik and Stuart, Andrew and Anandkumar, Anima},
  journal={arXiv preprint arXiv:2010.08895},
  year={2020}
}

@article{lu2021learning,
  title={Learning nonlinear operators via DeepONet based on the universal approximation theorem of operators},
  author={Lu, Lu and Jin, Pengzhan and Pang, Guofei and Zhang, Zhongqiang and Karniadakis, George Em},
  journal={Nature machine intelligence},
  volume={3},
  number={3},
  pages={218--229},
  year={2021},
  publisher={Nature Publishing Group UK London}
}

@article{kovachki2023neural,
  title={Neural operator: Learning maps between function spaces with applications to pdes},
  author={Kovachki, Nikola and Li, Zongyi and Liu, Burigede and Azizzadenesheli, Kamyar and Bhattacharya, Kaushik and Stuart, Andrew and Anandkumar, Anima},
  journal={Journal of Machine Learning Research},
  volume={24},
  number={89},
  pages={1--97},
  year={2023}
}

@article{liu2025efficient,
  title={An Efficient Graph-Transformer Operator for Learning Physical Dynamics with Manifolds Embedding},
  author={Liu, Pengwei and Ren, Xingyu and Wang, Pengkai and Yuan, Hangjie and Hao, Zhongkai and Chen, Guanyu and Xu, Chao and Ni, Dong and Cai, Shengze},
  journal={arXiv preprint arXiv:2512.10227},
  year={2025}
}

@inproceedings{liu2025aerogto,
  title={Aerogto: An efficient graph-transformer operator for learning large-scale aerodynamics of 3d vehicle geometries},
  author={Liu, Pengwei and Wang, Pengkai and Ren, Xingyu and Yuan, Hangjie and Hao, Zhongkai and Xu, Chao and Cai, Shengze and Ni, Dong},
  booktitle={Proceedings of the AAAI Conference on Artificial Intelligence},
  volume={39},
  number={18},
  pages={18924--18932},
  year={2025}
}

@article{takamoto2022pdebench,
  title={Pdebench: An extensive benchmark for scientific machine learning},
  author={Takamoto, Makoto and Praditia, Timothy and Leiteritz, Raphael and MacKinlay, Daniel and Alesiani, Francesco and Pfl{\"u}ger, Dirk and Niepert, Mathias},
  journal={Advances in neural information processing systems},
  volume={35},
  pages={1596--1611},
  year={2022}
}

@article{wu2024transolver,
  title={Transolver: A fast transformer solver for pdes on general geometries},
  author={Wu, Haixu and Luo, Huakun and Wang, Haowen and Wang, Jianmin and Long, Mingsheng},
  journal={arXiv preprint arXiv:2402.02366},
  year={2024}
}

@article{cao2021choose,
  title={Choose a transformer: Fourier or galerkin},
  author={Cao, Shuhao},
  journal={Advances in neural information processing systems},
  volume={34},
  pages={24924--24940},
  year={2021}
}

@article{gupta2021multiwavelet,
  title={Multiwavelet-based operator learning for differential equations},
  author={Gupta, Gaurav and Xiao, Xiongye and Bogdan, Paul},
  journal={Advances in neural information processing systems},
  volume={34},
  pages={24048--24062},
  year={2021}
}

@article{brandstetter2022message,
  title={Message passing neural PDE solvers},
  author={Brandstetter, Johannes and Worrall, Daniel and Welling, Max},
  journal={arXiv preprint arXiv:2202.03376},
  year={2022}
}

@article{pfaff2021meshgraphnets,
  title={Learning mesh-based simulation with graph networks},
  author={Pfaff, Tobias and Fortunato, Meire and Sanchez-Gonzalez, Alvaro and Battaglia, Peter W},
  journal={arXiv preprint arXiv:2010.03409},
  year={2020}
}

@article{boussif2022magnet,
  title={Magnet: Mesh agnostic neural pde solver},
  author={Boussif, Oussama and Bengio, Yoshua and Benabbou, Loubna and Assouline, Dan},
  journal={Advances in Neural Information Processing Systems},
  volume={35},
  pages={31972--31985},
  year={2022}
}

@article{tran2023factorized,
  title={Factorized fourier neural operators},
  author={Tran, Alasdair and Mathews, Alexander and Xie, Lexing and Ong, Cheng Soon},
  journal={arXiv preprint arXiv:2111.13802},
  year={2021}
}

@article{helwig2023group,
  title={Group equivariant fourier neural operators for partial differential equations},
  author={Helwig, Jacob and Zhang, Xuan and Fu, Cong and Kurtin, Jerry and Wojtowytsch, Stephan and Ji, Shuiwang},
  journal={arXiv preprint arXiv:2306.05697},
  year={2023}
}

@article{ren2026foundation,
  title={Foundation neural operators: A survey on pretraining methods, the data ecosystem, and efficient adaptation},
  author={Ren, Xingyu and Wang, Pengkai and Liu, Pengwei and Yue, Xihang and Dong, Huanshuo and Huang, Zhenxin and Hao, Zhongkai and Hu, Ziqian and Huang, Zhen and Wang, Yian and others},
  year={2026},
  publisher={TechRxiv}
}

@inproceedings{renuncertainty,
  title={Uncertainty-Informed Meta Pseudo Labeling for Surrogate Modeling with Limited Labeled Data},
  author={Ren, Xingyu and Liu, Pengwei and Wang, Pengkai and Chen, Guanyu and Wu, Qinxin and Ni, Dong},
  booktitle={The Thirty-ninth Annual Conference on Neural Information Processing Systems},
  year={2025}
}

@inproceedings{wu2025mpg,
  title={MPG: An Efficient Multi-scale Point-Based GNN for Non-uniform Meshes},
  author={Wu, Qinxin and Liu, Pengwei and Ren, Xingyu and Ni, Dong},
  booktitle={Joint European Conference on Machine Learning and Knowledge Discovery in Databases},
  pages={3--18},
  year={2025},
  organization={Springer}
}

@inproceedings{hao2023gnot,
  title={Gnot: A general neural operator transformer for operator learning},
  author={Hao, Zhongkai and Wang, Zhengyi and Su, Hang and Ying, Chengyang and Dong, Yinpeng and Liu, Songming and Cheng, Ze and Song, Jian and Zhu, Jun},
  booktitle={International conference on machine learning},
  pages={12556--12569},
  year={2023},
  organization={PMLR}
}

@article{he2024mgno,
  title={MgNO: Efficient parameterization of linear operators via multigrid},
  author={He, Juncai and Liu, Xinliang and Xu, Jinchao},
  journal={arXiv preprint arXiv:2310.19809},
  year={2023}
}

@article{wang2024latent,
  title={Latent neural operator for solving forward and inverse pde problems},
  author={Wang, Tian and Wang, Chuang},
  journal={Advances in Neural Information Processing Systems},
  volume={37},
  pages={33085--33107},
  year={2024}
}

@inproceedings{hao2024dpot,
  title={DPOT: auto-regressive denoising operator transformer for large-scale PDE pre-training},
  author={Hao, Zhongkai and Su, Chang and Liu, Songming and Berner, Julius and Ying, Chengyang and Su, Hang and Anandkumar, Anima and Song, Jian and Zhu, Jun},
  booktitle={Proceedings of the 41st International Conference on Machine Learning},
  pages={17616--17635},
  year={2024}
}

@article{li2023geofno,
  title={Fourier neural operator with learned deformations for pdes on general geometries},
  author={Li, Zongyi and Huang, Daniel Zhengyu and Liu, Burigede and Anandkumar, Anima},
  journal={Journal of Machine Learning Research},
  volume={24},
  number={388},
  pages={1--26},
  year={2023}
}

@article{li2023gino,
  title={Geometry-informed neural operator for large-scale 3d pdes},
  author={Li, Zongyi and Kovachki, Nikola and Choy, Chris and Li, Boyi and Kossaifi, Jean and Otta, Shourya and Nabian, Mohammad Amin and Stadler, Maximilian and Hundt, Christian and Azizzadenesheli, Kamyar and others},
  journal={Advances in Neural Information Processing Systems},
  volume={36},
  pages={35836--35854},
  year={2023}
}

@article{zheng2024alias,
  title={Alias-free mamba neural operator},
  author={Zheng, Jianwei and Li, Wei and Xu, Ni and Zhu, Junwei and Lin, Xiaoxu and Zhang, Xiaoqin},
  journal={Advances in Neural Information Processing Systems},
  volume={37},
  pages={52962--52995},
  year={2024}
}

@article{li2020neural,
  title={Neural operator: Graph kernel network for partial differential equations},
  author={Li, Zongyi and Kovachki, Nikola and Azizzadenesheli, Kamyar and Liu, Burigede and Bhattacharya, Kaushik and Stuart, Andrew and Anandkumar, Anima},
  journal={arXiv preprint arXiv:2003.03485},
  year={2020}
}

@article{cao2024laplace,
  title={Laplace neural operator for solving differential equations},
  author={Cao, Qianying and Goswami, Somdatta and Karniadakis, George Em},
  journal={Nature Machine Intelligence},
  volume={6},
  number={6},
  pages={631--640},
  year={2024},
  publisher={Nature Publishing Group UK London}
}

@article{azizzadenesheli2024neural,
  title={Neural operators for accelerating scientific simulations and design},
  author={Azizzadenesheli, Kamyar and Kovachki, Nikola and Li, Zongyi and Liu-Schiaffini, Miguel and Kossaifi, Jean and Anandkumar, Anima},
  journal={Nature Reviews Physics},
  volume={6},
  number={5},
  pages={320--328},
  year={2024},
  publisher={Nature Publishing Group UK London}
}

@article{koehler2024apebench,
  title={Apebench: A benchmark for autoregressive neural emulators of pdes},
  author={Koehler, Felix and Niedermayr, Simon and Westermann, R{\"u}diger and Thuerey, Nils},
  journal={Advances in Neural Information Processing Systems},
  volume={37},
  pages={120252--120310},
  year={2024}
}

@inproceedings{kurth2023fourcastnet,
  title={Fourcastnet: Accelerating global high-resolution weather forecasting using adaptive fourier neural operators},
  author={Kurth, Thorsten and Subramanian, Shashank and Harrington, Peter and Pathak, Jaideep and Mardani, Morteza and Hall, David and Miele, Andrea and Kashinath, Karthik and Anandkumar, Anima},
  booktitle={Proceedings of the platform for advanced scientific computing conference},
  pages={1--11},
  year={2023}
}

@article{goswami2022deeptransfer,
  title={Deep transfer operator learning for partial differential equations under conditional shift},
  author={Goswami, Somdatta and Kontolati, Katiana and Shields, Michael D and Karniadakis, George Em},
  journal={Nature Machine Intelligence},
  volume={4},
  number={12},
  pages={1155--1164},
  year={2022},
  publisher={Nature Publishing Group UK London}
}

@article{yin2024scalable,
  title={A scalable framework for learning the geometry-dependent solution operators of partial differential equations},
  author={Yin, Minglang and Charon, Nicolas and Brody, Ryan and Lu, Lu and Trayanova, Natalia and Maggioni, Mauro},
  journal={Nature computational science},
  volume={4},
  number={12},
  pages={928--940},
  year={2024},
  publisher={Nature Publishing Group US New York}
}

@article{zhao2025dno,
  title={Diffeomorphism neural operator for various domains and parameters of partial differential equations},
  author={Zhao, Zhiwei and Liu, Changqing and Li, Yingguang and Chen, Zhibin and Liu, Xu},
  journal={Communications Physics},
  volume={8},
  number={1},
  pages={15},
  year={2025},
  publisher={Nature Publishing Group UK London}
}

@article{herde2024poseidon,
  title={Poseidon: Efficient foundation models for pdes},
  author={Herde, Maximilian and Raoni{\'c}, Bogdan and Rohner, Tobias and K{\"a}ppeli, Roger and Molinaro, Roberto and De Bezenac, Emmanuel and Mishra, Siddhartha},
  journal={Advances in Neural Information Processing Systems},
  volume={37},
  pages={72525--72624},
  year={2024}
}

@article{rahman2024pretraining,
  title={Pretraining codomain attention neural operators for solving multiphysics pdes},
  author={Rahman, Ashiqur and George, Robert J and Elleithy, Mogab and Leibovici, Daniel and Li, Zongyi and Bonev, Boris and White, Colin and Berner, Julius and Yeh, Raymond A and Kossaifi, Jean and others},
  journal={Advances in Neural Information Processing Systems},
  volume={37},
  pages={104035--104064},
  year={2024}
}

@article{mccabe2024mpp,
  title={Multiple physics pretraining for spatiotemporal surrogate models},
  author={McCabe, Michael and R{\'e}galdo-Saint Blancard, Bruno and Parker, Liam and Ohana, Ruben and Cranmer, Miles and Bietti, Alberto and Eickenberg, Michael and Golkar, Siavash and Krawezik, Geraud and Lanusse, Francois and others},
  journal={Advances in Neural Information Processing Systems},
  volume={37},
  pages={119301--119335},
  year={2024}
}

@article{ohana2024thewell,
  title={The well: a large-scale collection of diverse physics simulations for machine learning},
  author={Ohana, Ruben and McCabe, Michael and Meyer, Lucas and Morel, Rudy and Agocs, Fruzsina J and Beneitez, Miguel and Berger, Marsha and Burkhart, Blakesley and Dalziel, Stuart B and Fielding, Drummond B and others},
  journal={Advances in Neural Information Processing Systems},
  volume={37},
  pages={44989--45037},
  year={2024}
}

@article{mccabe2025walrus,
  title={Walrus: A cross-domain foundation model for continuum dynamics},
  author={McCabe, Michael and Mukhopadhyay, Payel and Marwah, Tanya and Blancard, Bruno Regaldo-Saint and Rozet, Francois and Diaconu, Cristiana and Meyer, Lucas and Wong, Kaze WK and Sotoudeh, Hadi and Bietti, Alberto and others},
  journal={arXiv preprint arXiv:2511.15684},
  year={2025}
}

@article{bi2023pangu,
  title={Accurate medium-range global weather forecasting with 3D neural networks},
  author={Bi, Kaifeng and Xie, Lingxi and Zhang, Hengheng and Chen, Xin and Gu, Xiaotao and Tian, Qi},
  journal={Nature},
  volume={619},
  number={7970},
  pages={533--538},
  year={2023},
  publisher={Nature Publishing Group UK London}
}

@article{lam2023graphcast,
  title={Learning skillful medium-range global weather forecasting},
  author={Lam, Remi and Sanchez-Gonzalez, Alvaro and Willson, Matthew and Wirnsberger, Peter and Fortunato, Meire and Alet, Ferran and Ravuri, Suman and Ewalds, Timo and Eaton-Rosen, Zach and Hu, Weihua and others},
  journal={Science},
  volume={382},
  number={6677},
  pages={1416--1421},
  year={2023},
  publisher={American Association for the Advancement of Science}
}

@article{kochkov2024neuralgcm,
  title={Neural general circulation models for weather and climate},
  author={Kochkov, Dmitrii and Yuval, Janni and Langmore, Ian and Norgaard, Peter and Smith, Jamie and Mooers, Griffin and Kl{\"o}wer, Milan and Lottes, James and Rasp, Stephan and D{\"u}ben, Peter and others},
  journal={Nature},
  volume={632},
  number={8027},
  pages={1060--1066},
  year={2024},
  publisher={Nature Publishing Group UK London}
}

@article{bodnar2025aurora,
  title={A foundation model for the Earth system},
  author={Bodnar, Cristian and Bruinsma, Wessel P and Lucic, Ana and Stanley, Megan and Allen, Anna and Brandstetter, Johannes and Garvan, Patrick and Riechert, Maik and Weyn, Jonathan A and Dong, Haiyu and others},
  journal={Nature},
  volume={641},
  number={8065},
  pages={1180--1187},
  year={2025},
  publisher={Nature Publishing Group UK London}
}

@article{allen2025aardvark,
  title={End-to-end data-driven weather prediction},
  author={Allen, Anna and Markou, Stratis and Tebbutt, Will and Requeima, James and Bruinsma, Wessel P and Andersson, Tom R and Herzog, Michael and Lane, Nicholas D and Chantry, Matthew and Hosking, J Scott and others},
  journal={Nature},
  volume={641},
  number={8065},
  pages={1172--1179},
  year={2025},
  publisher={Nature Publishing Group UK London}
}

@inproceedings{wortsman2022model,
  title={Model soups: averaging weights of multiple fine-tuned models improves accuracy without increasing inference time},
  author={Wortsman, Mitchell and Ilharco, Gabriel and Gadre, Samir Ya and Roelofs, Rebecca and Gontijo-Lopes, Raphael and Morcos, Ari S and Namkoong, Hongseok and Farhadi, Ali and Carmon, Yair and Kornblith, Simon and others},
  booktitle={International conference on machine learning},
  pages={23965--23998},
  year={2022},
  organization={PMLR}
}

@article{matena2022merging,
  title={Merging models with fisher-weighted averaging},
  author={Matena, Michael S and Raffel, Colin A},
  journal={Advances in Neural Information Processing Systems},
  volume={35},
  pages={17703--17716},
  year={2022}
}

@article{ilharco2023editing,
  title={Editing models with task arithmetic},
  author={Ilharco, Gabriel and Ribeiro, Marco Tulio and Wortsman, Mitchell and Gururangan, Suchin and Schmidt, Ludwig and Hajishirzi, Hannaneh and Farhadi, Ali},
  journal={arXiv preprint arXiv:2212.04089},
  year={2022}
}

@article{yadav2023ties,
  title={Ties-merging: Resolving interference when merging models},
  author={Yadav, Prateek and Tam, Derek and Choshen, Leshem and Raffel, Colin A and Bansal, Mohit},
  journal={Advances in neural information processing systems},
  volume={36},
  pages={7093--7115},
  year={2023}
}

@inproceedings{yu2024dare,
  title={Language models are super mario: Absorbing abilities from homologous models as a free lunch},
  author={Yu, Le and Yu, Bowen and Yu, Haiyang and Huang, Fei and Li, Yongbin},
  booktitle={Forty-first International Conference on Machine Learning},
  year={2024}
}

@article{yang2024adamerging,
  title={Adamerging: Adaptive model merging for multi-task learning},
  author={Yang, Enneng and Wang, Zhenyi and Shen, Li and Liu, Shiwei and Guo, Guibing and Wang, Xingwei and Tao, Dacheng},
  journal={arXiv preprint arXiv:2310.02575},
  year={2023}
}

@article{yan2025calm,
  title={Calm: Consensus-aware localized merging for multi-task learning},
  author={Yan, Kunda and Zhang, Min and Cui, Sen and Qu, Zikun and Jiang, Bo and Liu, Feng and Zhang, Changshui},
  journal={arXiv preprint arXiv:2506.13406},
  year={2025}
}

@article{li2025map,
  title={Map: Low-compute model merging with amortized pareto fronts via quadratic approximation},
  author={Li, Lu and Zhang, Tianyu and Bu, Zhiqi and Wang, Suyuchen and He, Huan and Fu, Jie and Wu, Yonghui and Bian, Jiang and Chen, Yong and Bengio, Yoshua},
  journal={arXiv preprint arXiv:2406.07529},
  year={2024}
}

@article{huang2024emr,
  title={Emr-merging: Tuning-free high-performance model merging},
  author={Huang, Chenyu and Ye, Peng and Chen, Tao and He, Tong and Yue, Xiangyu and Ouyang, Wanli},
  journal={Advances in Neural Information Processing Systems},
  volume={37},
  pages={122741--122769},
  year={2024}
}

@article{ortizjimenez2023task,
  title={Task arithmetic in the tangent space: Improved editing of pre-trained models},
  author={Ortiz-Jimenez, Guillermo and Favero, Alessandro and Frossard, Pascal},
  journal={Advances in Neural Information Processing Systems},
  volume={36},
  pages={66727--66754},
  year={2023}
}

@article{singh2020model,
  title={Model fusion via optimal transport},
  author={Singh, Sidak Pal and Jaggi, Martin},
  journal={Advances in Neural Information Processing Systems},
  volume={33},
  pages={22045--22055},
  year={2020}
}

@article{ainsworth2023git,
  title={Git re-basin: Merging models modulo permutation symmetries},
  author={Ainsworth, Samuel K and Hayase, Jonathan and Srinivasa, Siddhartha},
  journal={arXiv preprint arXiv:2209.04836},
  year={2022}
}

@article{yang2023context,
  title={In-context operator learning with data prompts for differential equation problems},
  author={Yang, Liu and Liu, Siting and Meng, Tingwei and Osher, Stanley J},
  journal={Proceedings of the National Academy of Sciences},
  volume={120},
  number={39},
  pages={e2310142120},
  year={2023},
  publisher={National Academy of Sciences}
}

@article{hu2022lora,
  title={Lora: Low-rank adaptation of large language models.},
  author={Hu, Edward J and Shen, Yelong and Wallis, Phillip and Allen-Zhu, Zeyuan and Li, Yuanzhi and Wang, Shean and Wang, Liang and Chen, Weizhu and others},
  journal={Iclr},
  volume={1},
  number={2},
  pages={3},
  year={2022}
}

@article{wang2025model,
  title={Model merging scaling laws in large language models},
  author={Wang, Yuanyi and Gu, Yanggan and Zhang, Yiming and Zhou, Qi and Yan, Zhaoyi and Xie, Congkai and Wang, Xinyao and Yuan, Jianbo and Yang, Hongxia},
  journal={arXiv preprint arXiv:2509.24244},
  year={2025}
}

@article{gu2025infifpo,
  title={InfiFPO: Implicit model fusion via preference optimization in large language models},
  author={Gu, Yanggan and Wang, Yuanyi and Yan, Zhaoyi and Zhang, Yiming and Zhou, Qi and Wu, Fei and Yang, Hongxia},
  journal={arXiv preprint arXiv:2505.13878},
  year={2025}
}

@article{wang2025infigfusion,
  title={Infigfusion: Graph-on-logits distillation via efficient gromov-wasserstein for model fusion},
  author={Wang, Yuanyi and Yan, Zhaoyi and Zhang, Yiming and Zhou, Qi and Gu, Yanggan and Wu, Fei and Yang, Hongxia},
  journal={arXiv preprint arXiv:2505.13893},
  year={2025}
}

\clearpage
\appendix
\newcommand{\appendixline}[2]{%
  \par\noindent\hyperref[#2]{\textbf{#1}}%
  \nobreak\leaders\hbox to 0.6em{\hss.\hss}\hfill\textbf{\pageref{#2}}\par
}

\newcommand{\appendixsubline}[2]{%
  \par\noindent\hspace{2.2em}\hyperref[#2]{#1}%
  \nobreak\leaders\hbox to 0.6em{\hss.\hss}\hfill\pageref{#2}\par
}

\newcommand{\apptablefont}{%
  \small
  \setlength{\tabcolsep}{4.5pt}
  \renewcommand{\arraystretch}{1.04}
}

\newcommand{\appdensefont}{%
  \footnotesize
  \setlength{\tabcolsep}{3.5pt}
  \renewcommand{\arraystretch}{1.03}
}

\newcommand{\makeappendixcontents}{%
  \clearpage
  \phantomsection
  \section*{Appendix Contents}
  \label{sec:appendix}
  \vspace{0.7em}
  \begingroup
  \normalsize
  \setlength{\parskip}{0.18em}
  \setlength{\baselineskip}{1.08\baselineskip}

  \appendixline{A \quad Theory for Controlled PDE Merge}{sec:app:theory}
  \appendixsubline{A.1 \quad Two-point continuation bound}{sec:app:two-point-bound}
  \appendixsubline{A.2 \quad Why same-base endpoint updates carry directional information}{sec:app:same-base-direction}
  \appendixsubline{A.3 \quad PDE regularity behind the curvature assumption}{sec:app:pde-regularity}
  \vspace{0.35em}

  \appendixline{B \quad Detailed Related Work}{sec:app:related}
  \vspace{0.35em}

  \appendixline{C \quad Experimental Protocols and Hyperparameters}{sec:app:protocols}
  \appendixsubline{C.1 \quad Training and merge protocol}{sec:app:training}
  \appendixsubline{C.2 \quad CCM variants and future-only calibration}{sec:app:algorithm}
  \vspace{0.35em}

  \appendixline{D \quad Datasets, PDE Equations, and Family Construction}{sec:app:data}
  \appendixsubline{D.1 \quad Controlled family axes}{sec:app:data-axes}
  \appendixsubline{D.2 \quad DiffReact, NS2D, and RDB details}{sec:app:data-diffreact}
  \vspace{0.35em}

  \appendixline{E \quad Additional Evaluation Results}{sec:app:ablations}
  \appendixsubline{E.1 \quad DiffReact dense coordinate control}{sec:app:diffreact-dense}
  \appendixsubline{E.2 \quad Scale, backbone, and OOD validation}{sec:app:fno-scale}
  \appendixsubline{E.3 \quad NS2D and RDB calibration details}{sec:app:ns2d-policy}
  \appendixsubline{E.4 \quad Additional comparisons, uncertainty, seeds, and physics metrics}{sec:app:method-ablation}
  \vspace{0.35em}

  \appendixline{F \quad Qualitative Case Studies}{sec:app:case-studies}
  \appendixsubline{F.1 \quad Late-time medium-gap visualizations}{sec:app:case-studies}
  \endgroup
  \clearpage
}

\providecommand{\eps}{\varepsilon}

\newsavebox{\pdemergekeybox}
\newenvironment{keybox}[1]{%
  \par\medskip\noindent
  \begin{lrbox}{\pdemergekeybox}%
  \begin{minipage}{0.92\linewidth}%
  \textbf{#1.}\par\smallskip
}{%
  \end{minipage}%
  \end{lrbox}%
  \noindent\setlength{\fboxsep}{6pt}\fbox{\usebox{\pdemergekeybox}}%
  \par\medskip
}

\makeappendixcontents

\FloatBarrier
\section{Theory for Controlled PDE Merge}
\label{sec:app:theory}
In this section, we derive a local, conditionally valid estimate for the Controlled PDE Merge problem.
It should be interpreted as an a posteriori continuation bound applicable to experts sharing the same methodological lineage, rather than as a universal guarantee valid for arbitrary weight interpolation schemes.
The only quantities that enter into the estimate are the terminal (endpoint) approximation error and the curvature of the trajectory in function space.

\paragraph{Functional setting.}
Let $(\mathcal X,\rho)$ denote the input probability space and let $\mathcal Y_T$ be a Hilbert space of finite-horizon trajectories.  
We define
\begin{equation}
  \mathcal H_\rho := L^2(\mathcal X,\rho;\mathcal Y_T),
  \qquad
  \norm{A}_\rho^2 := \mathbb E_{a\sim\rho}
  \norm{A(a)}_{\mathcal Y_T}^2 .
\end{equation}
For a normalized physical coordinate $s$, we denote by $\mathcal S_s\in\mathcal H_\rho$ the exact (i.e., ground-truth) solution operator associated with the underlying PDE.  
The lower and upper endpoints of the normalized coordinate are fixed at $s=-1$ and $s=1$, respectively.

Under these settings, let $\theta_0$ denote the base checkpoint, and let $\theta_-$ and $\theta_+$ denote two endpoint experts that are both fine-tuned from this common base.  
We define the parameter offsets
\[
  \Delta_- := \theta_- - \theta_0,
  \qquad
  \Delta_+ := \theta_+ - \theta_0,
\]
and decompose them into their shared and directional components,
\[
  \Delta^{\rm sh} := \frac{\Delta_-+\Delta_+}{2},
  \qquad
  \Delta^{\rm dir} := \frac{\Delta_+-\Delta_-}{2}.
\]
For $\alpha$ in a fixed local interval $J\supset\{-1,1\}$, we then define the corresponding family of models
\[
  \mathcal F_\alpha := \mathcal F_{\theta(\alpha)} .
\]

\begin{keybox}{Endpoint-defined coordinate line}
\[
  \theta(\alpha)
  := \theta_0+\Delta^{\rm sh}+\alpha\Delta^{\rm dir}
  = \frac{1-\alpha}{2}\theta_-+\frac{1+\alpha}{2}\theta_+ .
\]
\end{keybox}

In this parameterization, $\alpha=0$ yields the arithmetic average of the two endpoints, $\alpha=\pm1$ recovers the original endpoint experts, and $|\alpha|>1$ corresponds to a local extrapolation beyond the endpoints. All constants appearing below are understood to be local with respect to the interval $J$.

\subsection{Two-point continuation bound}
\label{sec:app:two-point-bound}

For $\alpha \in J$, denote by
\[
  I_\alpha := {\rm conv}\{-1,1,\alpha\}
\]
the smallest closed interval containing the points $-1$, $1$, and $\alpha$.

\begin{lemma}[Hilbert-valued two-point interpolation]
\label{lem:hilbert-two-point}
Let $\mathcal H$ be a Hilbert space and let $\varphi:I_\alpha\to\mathcal H$ be a twice continuously differentiable mapping, i.e., of class $C^2$. Define
\[
  \ell_\varphi(\alpha)
  := \frac{1-\alpha}{2}\,\varphi(-1)
  +  \frac{1+\alpha}{2}\,\varphi(1).
\]
Then the following interpolation error estimate holds:
\[
  \bigl\|\varphi(\alpha)-\ell_\varphi(\alpha)\bigr\|_{\mathcal H}
  \le
  \frac{|\alpha^2-1|}{2}
  \sup_{t\in I_\alpha}\bigl\|\varphi''(t)\bigr\|_{\mathcal H} .
\]
\end{lemma}

\begin{proof}
Set $r \coloneqq \varphi(\alpha) - \ell_\varphi(\alpha)$. If $r = 0$, the claim follows trivially. Otherwise, define the normalized vector
\[
  z \coloneqq \frac{r}{\lVert r \rVert_{\mathcal H}}
\]
and consider the scalar-valued function
\[
  g(t) \coloneqq \langle \varphi(t), z \rangle_{\mathcal H},
\]
to which we apply the scalar Lagrange remainder formula with interpolation nodes $-1$ and $1$. Then there exists some $\xi \in I_\alpha$ such that
\[
  \langle r, z \rangle_{\mathcal H}
  = g(\alpha) - \frac{1 - \alpha}{2} g(-1) - \frac{1 + \alpha}{2} g(1)
  = \frac{\alpha^2 - 1}{2} g''(\xi).
\]
Since $\langle r, z \rangle_{\mathcal H} = \lVert r \rVert_{\mathcal H}$ by construction of $z$, and
\[
  \lvert g''(\xi) \rvert \le \lVert \varphi''(\xi) \rVert_{\mathcal H},
\]
the claimed inequality follows immediately.
\end{proof}

\begin{theorem}[A posteriori controlled merge estimate]
\label{thm:pde-merge-core}
Assume that
\[
  \mathcal S,\mathcal F\in C^2(J;\mathcal H_\rho),
  \qquad
  E(\alpha):=\mathcal F_\alpha-\mathcal S_\alpha .
\]
Define
\[
  \epsilon_-:=\norm{E(-1)}_\rho,
  \qquad
  \epsilon_+:=\norm{E(1)}_\rho,
  \qquad
  K_E(I_\alpha):=\sup_{t\in I_\alpha}\norm{E''(t)}_\rho .
\]
Then, for every $\alpha\in J$, one has the a posteriori estimate
\begin{keybox}{Core a posteriori continuation estimate}
\[
  \norm{\mathcal F_\alpha-\mathcal S_\alpha}_\rho
  \le
  \left|\frac{1-\alpha}{2}\right|\epsilon_-
  +
  \left|\frac{1+\alpha}{2}\right|\epsilon_+
  +
  \frac{|\alpha^2-1|}{2}K_E(I_\alpha).
\]
\end{keybox}
In particular, setting
\[
  K_F(I_\alpha):=\sup_{t\in I_\alpha}\norm{\mathcal F''_t}_\rho,
  \qquad
  K_S(I_\alpha):=\sup_{t\in I_\alpha}\norm{\mathcal S''_t}_\rho,
\]
yields the fully computable upper bound
\begin{keybox}{Computable curvature form}
\[
  \norm{\mathcal F_\alpha-\mathcal S_\alpha}_\rho
  \le
  \left|\frac{1-\alpha}{2}\right|\epsilon_-
  +
  \left|\frac{1+\alpha}{2}\right|\epsilon_+
  +
  \frac{|\alpha^2-1|}{2}
  \bigl(K_F(I_\alpha)+K_S(I_\alpha)\bigr).
\]
\end{keybox}
Furthermore, if the target parameter is $s\in J$ while the merge is performed at the coefficient
$\alpha\in J$, then the following stability estimate quantifies the penalty induced by
this coordinate mismatch:
\begin{keybox}{Penalty for coordinate mismatch}
\[
  \norm{\mathcal F_\alpha-\mathcal S_s}_\rho
  \le
  \norm{\mathcal F_\alpha-\mathcal S_\alpha}_\rho
  + L_S(I_{\alpha,s})|\alpha-s| .
\]
\end{keybox}
where
\[
  I_{\alpha,s}:={\rm conv}\{\alpha,s\},
  \qquad
  L_S(I_{\alpha,s}):=\sup_{t\in I_{\alpha,s}}\norm{\mathcal S'_t}_\rho .
\]
\end{theorem}

\begin{proof}
We apply Lemma~\ref{lem:hilbert-two-point} to the error curve
$E:J\to\mathcal H_\rho$. Its endpoint interpolant is given by
\[
  \ell_E(\alpha)
  = \frac{1-\alpha}{2}E(-1)+\frac{1+\alpha}{2}E(1).
\]
Consequently,
\[
  \norm{E(\alpha)-\ell_E(\alpha)}_\rho
  \le \frac{|\alpha^2-1|}{2}K_E(I_\alpha).
\]
By the triangle inequality, we obtain
\[
  \norm{\ell_E(\alpha)}_\rho
  \le
  \left|\frac{1-\alpha}{2}\right|\epsilon_-
  +
  \left|\frac{1+\alpha}{2}\right|\epsilon_+ .
\]
Combining these two inequalities yields the first estimate. The second estimate
follows from the bound
\[
  K_E(I_\alpha)\le K_F(I_\alpha)+K_S(I_\alpha).
\]
Finally, we have
\[
  \norm{\mathcal F_\alpha-\mathcal S_s}_\rho
  \le
  \norm{\mathcal F_\alpha-\mathcal S_\alpha}_\rho
  +
  \norm{\mathcal S_\alpha-\mathcal S_s}_\rho,
\]
and
\[
  \mathcal S_\alpha-\mathcal S_s=\int_s^\alpha \mathcal S'_t\,dt,
  \qquad
  \norm{\mathcal S_\alpha-\mathcal S_s}_\rho
  \le L_S(I_{\alpha,s})|\alpha-s|.
\]
\end{proof}

\begin{remark}[Second-order local form]
\label{rem:second-order-local}
Suppose the physical parameter is given by $\mu(s)=\mu_c+hs$ and that
\[
  \sup_{\mu}\norm{\partial_{\mu\mu}^2\mathcal S_\mu}_\rho\le \widetilde K_S .
\]
Then
\[
  K_S(I_\alpha)\le |h|^2\widetilde K_S .
\]
If, in addition, the induced neural path satisfies
\[
  K_F(I_\alpha)\le M_F\norm{\Delta^{\rm dir}}^2
  \quad\text{and}\quad
  \norm{\Delta^{\rm dir}}\le L_\theta |h|,
\]
then, for the coordinate-conditioned choice $\alpha=s$,
\begin{keybox}{Second-order local form}
\[
  \norm{\mathcal F_s-\mathcal S_s}_\rho
  \le
  \left|\frac{1-s}{2}\right|\epsilon_-
  +
  \left|\frac{1+s}{2}\right|\epsilon_+
  +
  \frac{|s^2-1|}{2}
  \bigl(M_F L_\theta^2+\widetilde K_S\bigr)|h|^2 .
\]
\end{keybox}

Consequently, modulo the endpoint error, the continuation error exhibits second-order scaling with respect to the half-gap between the physical
endpoints. 
The result is local in the sense that the constants involved are required only on the interval $J$ under consideration.
\end{remark}

\subsection{Why same-base endpoint updates carry directional information}
\label{sec:app:same-base-direction}

The previous estimate characterizes the function-space path after a particular trajectory has been fixed. We now formulate the minimal result that justifies why endpoint updates obtained from a common reference initialization induce a well-defined signed direction.

\begin{proposition}[Endpoint task vectors as finite differences]
\label{prop:task-vector-finite-difference}
Let $\Theta$ be a normed parameter space. Suppose that, in a neighborhood of the
central physical coordinate $q = 0$, the fine-tuning procedure selects a $C^3$ local expert branch
$q \mapsto \theta^\star(q) \in \Theta$. Define
\[
  U(q) := \theta^\star(q) - \theta_0.
\]
Assume that the trained endpoint parameters satisfy
\[
  \theta_- = \theta^\star(-h) + r_- ,
  \qquad
  \theta_+ = \theta^\star(h) + r_+ .
\]
Define the curvature bounds
\[
  M_2 := \sup_{|q|\le |h|}\norm{U''(q)},
  \qquad
  M_3 := \sup_{|q|\le |h|}\norm{U'''(q)} .
\]
Then the following quantitative finite-difference estimates hold.
\begin{keybox}{Endpoint updates approximate a central adaptation and a physical tangent}
\[
  \norm{\Delta^{\rm sh} - U(0)}
  \le
  \frac{|h|^2}{2} M_2 + \frac{\norm{r_-} + \norm{r_+}}{2},
\]
\[
  \norm{\Delta^{\rm dir} - h U'(0)}
  \le
  \frac{|h|^3}{6} M_3 + \frac{\norm{r_-} + \norm{r_+}}{2}.
\]
\end{keybox}
In particular, the shared update provides an approximation of the central adaptation $U(0)$, while the directional update provides an approximation of the first derivative $U'(0)$ of the local expert branch with respect to the physical parameter.

\end{proposition}

\begin{proof}
For the exact solution branch, Taylor's theorem yields
\[
  U(h) = U(0) + hU'(0) + \frac{h^2}{2}U''(0) + R_+,
  \qquad
  U(-h) = U(0) - hU'(0) + \frac{h^2}{2}U''(0) + R_-,
\]
where the integral form of the remainder implies the symmetric bound
\[
  \left\|\frac{U(h) + U(-h)}{2} - U(0)\right\|
  \le \frac{|h|^2}{2} M_2,
\]
and, analogously, the third-order remainder yields
\[
  \left\|\frac{U(h) - U(-h)}{2} - hU'(0)\right\|
  \le \frac{|h|^3}{6} M_3.
\]
The endpoint residuals contribute linearly through the relations
\[
  \Delta^{\rm sh} = \frac{U(-h) + U(h)}{2} + \frac{r_- + r_+}{2},
  \qquad
  \Delta^{\rm dir} = \frac{U(h) - U(-h)}{2} + \frac{r_+ - r_-}{2}.
\]
\end{proof}

\begin{remark}[A sufficient condition in a local chart]
\label{rem:local-chart}
Let $\mathcal L(q,\theta)$ denote a regularized population fine-tuning objective expressed in a local chart of the parameter space. 
Suppose that $\mathcal L\in C^{r+1}$ and that, after fixing a local gauge, restricting to the chart selected by the optimizer, or introducing an appropriate local regularization, the Hessian $\nabla_{\theta\theta}^2\mathcal L(0,\theta^\star(0))$ is invertible. 
Then, by the implicit function theorem, there exists a $C^r$ solution branch $q\mapsto\theta^\star(q)$. 
In particular, the $C^3$ regularity of the branch used in Proposition~\ref{prop:task-vector-finite-difference} is guaranteed by the assumption $\mathcal L\in C^4$ together with the same nondegeneracy condition.
\end{remark}

\subsection{PDE regularity behind the curvature assumption}
\label{sec:app:pde-regularity}

We finally give a standard sufficient condition that guarantees twice continuously differentiable dependence of the associated solution operator of the PDE. 
This regularity requirement is not imposed on the neural network itself; instead, it constitutes a local well-posedness assumption on the underlying family of partial differential equations.

\begin{proposition}[Smooth dependence of the PDE solution operator]
\label{prop:pde-smooth-dependence}
Let the PDE be written as a residual equation
\[
  \mathfrak R(u,a,\mu)=0,
\]
where $u\in\mathcal U$, $\mathfrak R(u,a,\mu)\in\mathcal Z$, and
$\mathcal U,\mathcal Z$ are Banach spaces continuously embedded into the
trajectory norm used in $\mathcal Y_T$.  Suppose that, for $\rho$-a.e. input
$a$ and all $\mu$ in a neighborhood of the interval of interest:
\begin{enumerate}
  \item $\mathfrak R$ is $C^2$ in $(u,\mu)$;
  \item there is a solution $u(a,\mu)$ with $\mathfrak R(u(a,\mu),a,\mu)=0$;
  \item $D_u\mathfrak R(u(a,\mu),a,\mu):\mathcal U\to\mathcal Z$ is invertible
  with uniformly bounded inverse;
  \item the first and second derivatives of $\mathfrak R$ entering the
  sensitivity equations are bounded by a $\rho$-square-integrable envelope.
\end{enumerate}
Then $\mu\mapsto\mathcal S_\mu\in\mathcal H_\rho$ is $C^2$.  Moreover,
\[
  \sup_\mu \norm{\partial_{\mu\mu}^2\mathcal S_\mu}_\rho <\infty .
\]
Consequently, for $\mu(s)=\mu_c+hs$,
\[
  \frac{d^2}{ds^2}\mathcal S_{\mu(s)}
  = h^2\partial_{\mu\mu}^2\mathcal S_{\mu(s)},
  \qquad
  K_S\lesssim |h|^2 .
\]
\end{proposition}

\begin{proof}
For each admissible input $a$, the implicit function theorem in Banach spaces ensures the local $C^2$-regularity of the mapping $\mu \mapsto u(a,\mu)$. Differentiating the identity
\[
  \mathfrak R(u(a,\mu),a,\mu)=0
\]
with respect to $\mu$ yields the first-order sensitivity relation
\[
  D_u\mathfrak R\, u_\mu = -D_\mu\mathfrak R,
\]
and, upon differentiating once more with respect to $\mu$, the second-order sensitivity equation
\[
  D_u\mathfrak R\, u_{\mu\mu}
  = -D_{uu}^2\mathfrak R[u_\mu,u_\mu]
    -2D_{u\mu}^2\mathfrak R[u_\mu]
    -D_{\mu\mu}^2\mathfrak R .
\]
The assumed uniform invertibility of $D_u\mathfrak R$ together with the uniform bounds on the derivatives of $\mathfrak R$ implies pointwise estimates on $u_\mu$ and $u_{\mu\mu}$, which are dominated by the prescribed square-integrable envelope. 
By the dominated convergence theorem, these pointwise differentiability properties imply that the mapping $\mu \mapsto \mathcal S_\mu$ is in fact twice continuously differentiable as a map from the parameter space into $L^2(\mathcal X,\rho;\mathcal Y_T)$.
\end{proof}

\begin{remark}[Nonsmooth architectures and empirical curvature]
If the mapping $\alpha \mapsto \mathcal F_\alpha$ is not known to be twice continuously differentiable ($C^2$), for instance due to nonsmooth activation functions or discrete algorithmic components, the functional curvature term in Theorem~\ref{thm:pde-merge-core} can be approximated via the second finite difference
\[
  \widehat K_F(\alpha,\delta)
  :=
  \frac{\lVert \mathcal F_{\alpha+\delta} - 2\mathcal F_\alpha
  + \mathcal F_{\alpha-\delta} \rVert_\rho}{\delta^2} .
\]
This quantity constitutes the discrete analogue of $K_F$ and coincides with what is naturally accessed through $\alpha$-sweep diagnostics. 
Accordingly, the mathematical bound above is most appropriately interpreted as a local continuation estimate, whose neural-curvature assumption ought to be validated or systematically ablated through empirical investigation.
\end{remark}

\paragraph{Scope.}
The theorem is formulated on a fixed local interval $J$. 
It does not address phenomena such as breakdown of PDE well-posedness, the formation of shocks or bifurcations, endpoint experts trained in distinct attraction basins, or merge trajectories exhibiting large functional curvature. 
Under the given local regularity and compatibility conditions, endpoint experts that arise from the same initialization basin produce finite-difference estimates of a physically interpretable tangent direction, and the merged model follows the corresponding trajectory in the space of PDE operators with a quantitatively controlled error in the associated two-point continuation.

\FloatBarrier

\section{Detailed Related Work}
\label{sec:app:related}

\paragraph{Neural operators and PDE benchmarks.}
Neural operators learn mappings between spaces of functions, offering a unified paradigm for data-driven surrogate models of parametric PDEs.
The Fourier Neural Operator introduced spectral convolution as an efficient operator layer for regular grids, while DeepONet proposed an operator-learning architecture using separate branch and trunk networks~\cite{li2021fourier,lu2021learning,kovachki2023neural}.
Since then, the class of architectures has grown in many directions, including graph-kernel and mesh-based operators, multiwavelet and factorized spectral layers, Galerkin- or transformer-style attention mechanisms, message-passing-based solvers, multigrid parameterizations, latent-space operators, and fast transformer architectures tailored to PDEs
~\cite{li2020neural,pfaff2021meshgraphnets,gupta2021multiwavelet,
cao2021choose,tran2023factorized,brandstetter2022message,hao2023gnot,
he2024mgno,wang2024latent,wu2024transolver}.
In parallel, benchmark collections such as PDEBench and APEBench have refocused evaluation on families of dynamical systems, multi-step rollout horizons, and controlled variations in the governing parameters
~\cite{takamoto2022pdebench,koehler2024apebench}.

\paragraph{PDE foundation models and large physical surrogates.}
Recent work has also started to develop neural operators and spatiotemporal surrogate models that span broader families of PDEs, physical quantities, and simulation settings.
DPOT~\cite{hao2024dpot}, Poseidon~\cite{herde2024poseidon}, multiple-physics pretraining~\cite{rahman2024pretraining}, the Well~\cite{ohana2024thewell}, and Walrus~\cite{mccabe2025walrus} exemplify a shift toward reusable models and datasets for continuum dynamics.
Parallel advances in weather and climate modeling demonstrate that learned physical surrogates can serve as large-scale, long-range forecasting systems~\cite{kurth2023fourcastnet,bi2023pangu,lam2023graphcast,
kochkov2024neuralgcm,bodnar2025aurora,allen2025aardvark}.
Together, these developments underscore the importance of efficient adaptation once a solver has already been trained.

\paragraph{Conditioning, geometry, and adaptation for PDE families.}
A common strategy for handling parameter changes is to condition the solver on metadata, geometric information, prompts, or other domain-aware encodings.
Geo-FNO~\cite{li2023geofno}, GINO~\cite{li2023gino}, and diffeomorphism neural operators~\cite{zhao2025dno} all alter how functions or domains are represented so that the learned operator can generalize across different geometries and parameter settings.
Mesh-agnostic and symmetry-aware solvers follow a related design philosophy, modifying the state representation and the message-passing or equivariant operator components~\cite{pfaff2021meshgraphnets,boussif2022magnet,helwig2023group,wu2025mpg}.
Transfer operator learning adapts a source operator to a shifted target distribution via additional optimization, whereas in-context operator learning conditions on auxiliary examples or prompts at inference time~\cite{goswami2022deeptransfer,yang2023context}.
Parameter-efficient fine-tuning techniques such as low-rank adaptation lower the storage cost of keeping many adapted models, but still introduce target-specific parameter updates~\cite{hu2022lora}.
CCM is complementary to these approaches: rather than altering the architecture or retraining for each target regime, it composes pre-trained endpoint experts and uses physical metadata or a brief observed prefix to select a coordinate within the induced family of checkpoints.

\paragraph{Model merging and task vectors.}
Model merging studies when separately fine-tuned models, typically starting from a shared initialization, can be combined directly in weight space.
Model soups average the fine-tuned parameters, Fisher merging incorporates curvature information, and task arithmetic interprets fine-tuning updates as vectors that can be added or subtracted~\cite{wortsman2022model,matena2022merging,ilharco2023editing}.
Subsequent approaches tackle issues such as interference, permutation symmetries, tangent-space formulations, adaptive task weighting, and more scalable merging schemes~\cite{singh2020model,ainsworth2023git,yadav2023ties,ortizjimenez2023task,
yang2024adamerging,yu2024dare,huang2024emr,yan2025calm,
li2025map}.
In nearly all of this work, tasks are treated as discrete, unordered entities, and the merge coefficient is typically chosen via validation performance or a fixed algorithmic heuristic.
PDE families, in contrast, introduce a different kind of structure: tasks are indexed by continuous physical parameters such as diffusivity, viscosity, or geometry.
This enables a direct comparison between a loss-optimized merging coefficient and an external physical coordinate, which is precisely the diagnostic role of the $\alpha$-coordinate experiments in the main text.

\FloatBarrier
\section{Experimental Protocols and Hyperparameters}
\label{sec:app:protocols}

\paragraph{Evaluation conventions.}
The appendix adheres to the main-text configuration of \emph{family-axis} composition for shared-base neural PDE experts. 
Out-of-distribution evaluation corresponds to extrapolation along the physical axis that parameterizes the endpoint experts. 
For CCM-Prefix, the observed prefix segment of the trajectory is used to determine the coefficient, and all reported future-rollout L2 errors are computed over the subsequent prediction horizon.

\subsection{Training and merge protocol}
\label{sec:app:training}

The final reported experiments employ a best-checkpoint lineage: experts are initialized from the corresponding family-level \texttt{best} checkpoint.  
For RDB, the reported endpoint branch is selected according to the local-chart criterion described below, characterized by a low learning rate, the absence of additional noise, and the use of support--anchor mixing.  
NS2D and DiffReact/DPOT experts are likewise initialized from their respective base-best checkpoints.

Training is conducted using either rollout-aware or random-time autoregressive supervision, depending on the specific benchmark. Evaluation metrics are reported as full-rollout or future-rollout L2 errors over the prescribed temporal horizons.  
Normalizer configurations and data-scaling procedures are held fixed by the training and evaluation protocol.

All reported training runs, fine-tuning procedures, merge-sweep experiments, and evaluation jobs were executed on a single server equipped with eight NVIDIA A800 GPUs. 
When multiple sweeps were required, independent single-GPU jobs were parallelized across the available devices.

\begin{table}[!htbp]
\centering
\apptablefont
\caption{Training settings for the final reported model families.}
\label{tab:app-training-settings}
\begin{tabular}{p{0.18\linewidth}p{0.31\linewidth}p{0.39\linewidth}}
\toprule
\textbf{Model family} & \textbf{Family-base setting} & \textbf{Endpoint-expert setting} \\
\midrule
RDB FNO & width 64, modes 16, lr $5{\times}10^{-4}$ & lr $2{\times}10^{-6}$, no noise, support--anchor mixing \\
NS2D FNO & width 64, modes 16, lr $10^{-4}$ & lr $5{\times}10^{-5}$, no noise \\
DPOT DiffReact & width 512, lr $5{\times}10^{-4}$ & lr $10^{-4}$, no noise \\
FNO scale & small/base/large FNO variants & endpoint experts initialized from base-best checkpoints \\
\bottomrule
\end{tabular}
\end{table}

\paragraph{Locality criterion for endpoint branches.}
The reported endpoint branches are selected to satisfy local-chart stability
around the shared base.
For RDB, the criterion uses conservative endpoint gaps, no added training noise,
a low expert learning rate, and support--anchor mixing.
These choices keep endpoint updates in the regime where the linear
shared/directional decomposition is meaningful.

\FloatBarrier
\subsection{CCM algorithmic details}
\label{sec:app:algorithm}

Given base checkpoint $\theta_0$ and endpoint experts
$\theta_{\rm low},\theta_{\rm high}$, all CCM variants share the same weight
family:
\[
\Delta_{\rm low}=\theta_{\rm low}-\theta_0,\qquad
\Delta_{\rm high}=\theta_{\rm high}-\theta_0,
\]
\[
\Delta^{+}=(\Delta_{\rm low}+\Delta_{\rm high})/2,\qquad
\Delta^{-}=(\Delta_{\rm high}-\Delta_{\rm low})/2,
\]
\[
\theta(\alpha)=\theta_0+\Delta^{+}+\alpha\Delta^{-}.
\]
The selection mechanism varies according to the information condition.
CCM-Coord uses the known normalized coordinate $s(\lambda)$; CCM-Scale estimates a single scaling parameter $\gamma$; and CCM-Prefix performs a search over a finite set of candidate $\alpha$ values using a short rollout prefix.

\begin{table}[!htbp]
\centering
\apptablefont
\caption{Coefficient-selection rules for endpoint composition.}
\label{tab:app-ccm-variants}
\resizebox{\linewidth}{!}{\begin{tabular}{llll}
\toprule
\textbf{Rule} & \textbf{Input information} & \textbf{Coefficient} & \textbf{Checkpoint used for rollout} \\
\midrule
Endpoint average & none & $\alpha=0$ & one composed checkpoint \\
Best fixed $\alpha$ & validation sweep & one validation-selected scalar & one composed checkpoint \\
CCM-Coord & known coordinate $s(\lambda)$ & $\alpha=s(\lambda)$ & one composed checkpoint \\
CCM-Scale & calibrated scale $\gamma$ & $\alpha=\gamma s(\lambda)$ & one composed checkpoint \\
CCM-Prefix & $K$ observed rollout-prefix steps & prefix-selected bank value & one composed checkpoint \\
Oracle $\alpha$ & full target rollout & per-task best bank value & one composed checkpoint \\
\bottomrule
\end{tabular}}
\end{table}

\FloatBarrier
\subsection{Future-only CCM-Prefix protocol for RDB}
\label{sec:app:rdb-protocol}

RDB high-center uses a strict future-only protocol with
$\mathcal{I}_{\mathrm{cal}}=\{1,\ldots,4\}$ and
$\mathcal{I}_{\mathrm{fut}}=\{5,\ldots,T\}$.
The $\alpha$ bank is
\[
\mathcal{A}=\{-1.50,-1.25,-1.00,-0.75,-0.50,-0.25,0,
0.25,0.50,0.75,1.00,1.25,1.50\}.
\]
For each test sample and each candidate $\alpha$, the model rolls out only the first four calibration steps.
The selected $\alpha$ is the candidate with smallest prefix-selection L2.
The reported score is future-rollout L2 computed on $\mathcal{I}_{\mathrm{fut}}$.
This separates coefficient selection from the frames used for future-rollout scoring.

\paragraph{Interpretable selected $\alpha$ values.}
The selected $\alpha$ values are physically interpretable at endpoint and OOD regimes:
OOD-low and medium-gap low tasks select negative coefficients, while medium-gap high and OOD-high select positive coefficients.
Interior support tasks have higher variance, making RDB a prefix-calibrated case with sample-specific coefficient selection.

\FloatBarrier
\section{Datasets, PDE Equations, and Family Construction}
\label{sec:app:data}

\subsection{Controlled family axes}
\label{sec:app:data-axes}

All benchmarks use controlled one-dimensional family axes.
DiffReact varies one coefficient at a time and includes sparse small/medium gap tasks plus a dense $D_u$ evaluation axis.
NS2D varies viscosity around a support family with endpoint, medium, and OOD viscosity values.
RDB varies the inner dam-break height in the high-center setting with support, endpoint, and OOD low/high tasks.
For all domains, base training uses support regimes, endpoint experts are fine-tuned from the same best base checkpoint, and all reported merges are post hoc.

\paragraph{Why controlled axes matter.}
Each benchmark is organized around an interpretable physical coordinate, so an
$\alpha$ selected in weight space can be compared with an independent family axis.
This structure distinguishes the experiments from unordered task-merging benchmarks.

\subsection{Diffusion--reaction family}
\label{sec:app:data-diffreact}

DiffReact uses the PDEBench two-channel diffusion--reaction simulator:
\[
\partial_t u=D_u\Delta u+u-u^3-k-v,\qquad
\partial_t v=D_v\Delta v+u-v.
\]
The anchor task uses $D_u=10^{-3}$, $D_v=5\times10^{-3}$, and
$k=5\times10^{-3}$.
The $f_1$, $f_2$, and $f_3$ families vary $k$, $D_u$, and $D_v$ respectively
while holding the other two coefficients fixed.
For dense $f_2$ evaluation, the coordinate is normalized as
\[
s(D_u)=
\begin{cases}
\dfrac{D_u-D_{u,0}}{D_{u,\mathrm{high}}-D_{u,0}}, & D_u\geq D_{u,0},\\[1ex]
-\dfrac{D_{u,0}-D_u}{D_{u,0}-D_{u,\mathrm{low}}}, & D_u<D_{u,0},
\end{cases}
\]
with $D_{u,0}=10^{-3}$, $D_{u,\mathrm{low}}=8\times10^{-4}$, and
$D_{u,\mathrm{high}}=1.2\times10^{-3}$.
Thus $D_u(s)=10^{-3}+2\times10^{-4}s$ on the core dense grid.
The dense DiffReact axis is used as an evaluation/export grid.
The exported metadata records \texttt{train\_size=1} for exported tasks; the base and endpoint experts are trained on the sparse support and endpoint regimes.

\subsection{FourierFlow NS2D viscosity family}
\label{sec:app:data-ns2d}

NS2D is a periodic vorticity-form incompressible Navier--Stokes family on a $128\times128$ torus.
The state is vorticity $\omega$, with fixed forcing and viscosity $\nu$ as the family coordinate:
\[
\partial_t \omega + \mathbf{v}\cdot\nabla\omega
=
\nu\Delta\omega + f,\qquad
\nabla\cdot \mathbf{v}=0 .
\]
The center viscosity is $\nu_0=10^{-4}$.
The viscosity ratios are
\[
\begin{aligned}
\nu/\nu_0 &\in \{0.60,0.80,1.00,1.20,1.50\}
&&\text{for support and endpoint tasks},\\
\nu/\nu_0 &\in \{0.40,0.50,1.80,2.20\}
&&\text{for medium evaluation tasks},\\
\nu/\nu_0 &\in \{0.25,3.00\}
&&\text{for OOD evaluation tasks}.
\end{aligned}
\]
The medium and OOD regimes are evaluation-only in the exported metadata.
All tasks share seed banks so that a fixed sample id corresponds to the same initial condition and differs only by viscosity.

\subsection{Radial dam-break high-center family}
\label{sec:app:data-rdb}

RDB uses the PDEBench radial dam-break shallow-water simulator and reports height-channel rollout error.
The underlying shallow-water equations can be written as
\[
\partial_t h+\nabla\cdot(h\mathbf{u})=0,\qquad
\partial_t \mathbf{u}+\mathbf{u}\cdot\nabla\mathbf{u}+g\nabla h=0 .
\]
The family coordinate is the inner water-column height, while the outer height is fixed at $1.0$.
The canonical high-center family uses center $3.0$, base supports $2.0$ and $4.2$, endpoint experts $1.7$ and $4.8$, and OOD heights $1.05$ and $6.6$.
This range keeps all tasks in the valid regime $h_{\mathrm{inner}}>1.0$ while leaving enough room on both sides to test endpoint directionality and local-chart stability.

\begin{table}[!htbp]
\centering
\apptablefont
\caption{Dataset construction by family.}
\label{tab:app-data-summary}
\begin{tabular}{llll}
\toprule
\textbf{Family} & \textbf{Train support} & \textbf{Endpoint experts} & \textbf{Eval-only regimes} \\
\midrule
DiffReact $f_2$ & center and small-gap support & low/high $D_u$ & dense ID/OOD grid \\
NS2D & center, midlow, midhigh & low/high viscosity & medium and OOD viscosity \\
RDB high-center & center, midlow, midhigh & low/high inner height & OOD low/high height \\
\bottomrule
\end{tabular}
\end{table}

\FloatBarrier
\section{Additional Evaluation Results}
\label{sec:app:ablations}

The main text uses Fig.~\ref{fig:cross-domain-evidence-atlas} as a compact atlas of the cross-domain evidence. This appendix keeps the complementary numeric tables, additional ablations, and case visualizations rather than repeating every single-panel plot from the main text.

\subsection{DiffReact dense coordinate control}
\label{sec:app:diffreact-dense}

The dense $D_u$ axis directly tests whether $\alpha$ tracks a physical family coordinate.
The selected coefficient varies systematically with the normalized physical coordinate, and the coordinate-aware policy matches the family-axis oracle on the exported dense tasks.
The main text reports the corresponding OOD reduction from $0.1877$ to $0.0861$.
Figure~\ref{fig:coordinate-loss-strip} presents the dense coordinate law and the corresponding loss profile together with the NS2D and RDB comparisons.
The table below keeps the complementary negative control.

\begin{table}[!htbp]
\centering
\appdensefont
\caption{Wrong-sign coordinate control on dense DiffReact.
The control uses the same $\alpha$ bank and checkpoints as CCM-Coord but flips the
coordinate from $\alpha=s$ to $\alpha=-s$; errors increase across all task
groups.}
\label{tab:app-wrong-coordinate-control}
\begin{tabular}{lcccc}
\toprule
\textbf{Task group} & \textbf{\# tasks} & \textbf{CCM-Coord} & \textbf{Wrong sign} & \textbf{Wrong / Coord} \\
\midrule
Interpolation & 7 & \sci{8.03}{-2} & \sci{1.36}{-1} & 1.69x \\
Endpoint & 2 & \sci{8.22}{-2} & \sci{2.49}{-1} & 3.03x \\
OOD & 4 & \sci{8.61}{-2} & \sci{3.36}{-1} & 3.91x \\
All & 13 & \sci{8.24}{-2} & \sci{2.15}{-1} & 2.61x \\
\bottomrule
\end{tabular}

\end{table}

\FloatBarrier
\subsection{FNO scale robustness}
\label{sec:app:fno-scale}

The FNO scale sweep evaluates merge behavior across model capacity.
Average merge improves over the base at small, base, and large FNO scales.
The relative gain increases with model scale, suggesting that stronger family bases make endpoint updates more composable.

\begin{table}[!htbp]
\centering
\appdensefont
\caption{Full FNO scale sweep on DiffReact $f_3$.}
\label{tab:app-f3-scale-sweep}
\begin{tabular}{llrrrr}
\toprule
\textbf{Scale} & \textbf{Method} & \textbf{Overall} & \textbf{Endpoint} & \textbf{Worst} & \textbf{Rel. gain} \\
\midrule
small & Base & \sci{2.12}{-1} & \sci{2.24}{-1} & \sci{2.33}{-1} & 0\% \\
small & Expert low & \sci{1.98}{-1} & \sci{2.05}{-1} & \sci{2.50}{-1} & 6.4\% \\
small & Expert high & \sci{2.30}{-1} & \sci{2.40}{-1} & \sci{3.22}{-1} & -8.7\% \\
small & Endpoint average & \sci{1.92}{-1} & \sci{2.04}{-1} & \sci{2.10}{-1} & 9.6\% \\
base & Base & \sci{1.87}{-1} & \sci{2.04}{-1} & \sci{2.32}{-1} & 0\% \\
base & Expert low & \sci{1.68}{-1} & \sci{1.75}{-1} & \sci{2.35}{-1} & 10.2\% \\
base & Expert high & \sci{1.84}{-1} & \sci{1.93}{-1} & \sci{2.80}{-1} & 1.8\% \\
base & Endpoint average & \sci{1.49}{-1} & \sci{1.65}{-1} & \sci{1.71}{-1} & 20.2\% \\
large & Base & \sci{1.50}{-1} & \sci{1.68}{-1} & \sci{1.75}{-1} & 0\% \\
large & Expert low & \sci{1.35}{-1} & \sci{1.41}{-1} & \sci{2.15}{-1} & 9.9\% \\
large & Expert high & \sci{1.48}{-1} & \sci{1.56}{-1} & \sci{2.48}{-1} & 1.1\% \\
large & Endpoint average & \sci{1.13}{-1} & \sci{1.32}{-1} & \sci{1.33}{-1} & 24.8\% \\
\bottomrule
\end{tabular}

\end{table}

\FloatBarrier
\subsection{Large-FNO merge baselines}
\label{sec:app:fno-large-baselines}

The large FNO line includes generic merge heuristics on the same checkpoints.
On this merge-friendly PDE axis, endpoint averaging gives lower error than task-arithmetic-style extrapolation.
This comparison separates the shared/directional construction from alternative task-vector edits on the same checkpoints.
All generic merge baselines in this comparison use the same base and endpoint checkpoints as the controlled merge methods.
Their coefficients, trimming/drop settings, and merge choices are selected from validation records under the same rule that excludes target-test labels for the other post-hoc policies.
They are non-conditional reference controls for scalar family-axis continuation.

\begin{table}[!htbp]
\centering
\appdensefont
\caption{Large-FNO DiffReact $f_3$ merge baselines.}
\label{tab:app-f3-large-baselines}
\begin{tabular}{lrrrr}
\toprule
\textbf{Method} & \textbf{Overall} & \textbf{Endpoint} & \textbf{Worst} & \textbf{mg-low} \\
\midrule
Base & \sci{1.499}{-1} & \sci{1.678}{-1} & \sci{1.746}{-1} & \sci{1.746}{-1} \\
Endpoint average & \sci{1.127}{-1} & \sci{1.320}{-1} & \sci{1.332}{-1} & \sci{1.332}{-1} \\
Task arithmetic & \sci{4.579}{3} & \sci{4.583}{3} & \sci{4.764}{3} & \sci{4.401}{3} \\
TIES & \sci{2.730}{4} & \sci{2.713}{4} & \sci{2.770}{4} & \sci{2.757}{4} \\
DARE & \sci{2.983}{2} & \sci{2.977}{2} & \sci{3.087}{2} & \sci{2.868}{2} \\
\bottomrule
\end{tabular}

\end{table}

\FloatBarrier
\begin{figure}[!htbp]
\centering
\includegraphics[width=\linewidth,height=0.50\textheight,keepaspectratio]{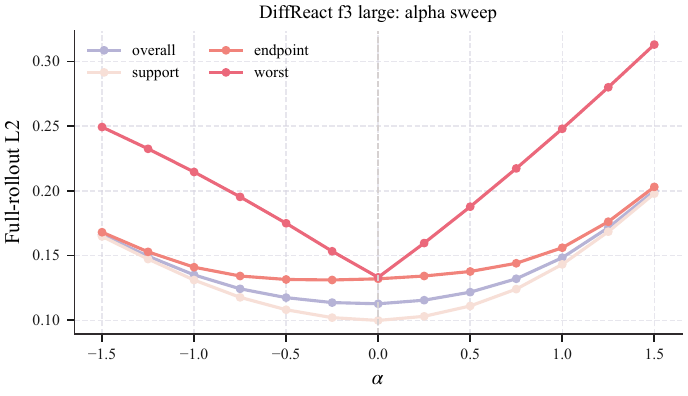}
\caption{Large-FNO DiffReact $f_3$ $\alpha$ sweep.}
\label{fig:app-f3-large-alpha-sweep}
\end{figure}

\FloatBarrier
\subsection{DPOT-style backbone validation}
\label{sec:app:dpot}

The DPOT-style line evaluates the same composition rule outside the FNO architecture.
The base reaches $0.0447$, average merge improves to $0.0391$, and CCM-Coord reaches $0.0183$, matching the oracle on this small coordinate bank.

\begin{table}[!htbp]
\centering
\appdensefont
\caption{DPOT-style DiffReact $f_3$ validation table.}
\label{tab:app-dpot-f3}
\begin{tabular}{lrrrr}
\toprule
\textbf{Method} & \textbf{Overall $\downarrow$} & \textbf{Support $\downarrow$} & \textbf{Endpoint $\downarrow$} & \textbf{Worst $\downarrow$} \\
\midrule
Base & \sci{4.47}{-2} & \sci{4.09}{-2} & \sci{5.04}{-2} & \sci{5.67}{-2} \\
Expert low & \sci{6.02}{-2} & \sci{5.94}{-2} & \sci{6.16}{-2} & \sci{1.043}{-1} \\
Expert high & \sci{5.83}{-2} & \sci{5.41}{-2} & \sci{6.46}{-2} & \sci{1.117}{-1} \\
Best fixed & \sci{3.91}{-2} & \sci{2.76}{-2} & \sci{5.63}{-2} & \sci{5.76}{-2} \\
Best endpoint fixed & \sci{4.10}{-2} & \sci{3.10}{-2} & \sci{5.61}{-2} & \sci{6.75}{-2} \\
Task arithmetic & \sci{7.35}{-2} & \sci{6.66}{-2} & \sci{8.39}{-2} & \sci{8.67}{-2} \\
CCM-Coord & \sci{1.83}{-2} & \sci{1.85}{-2} & \sci{1.81}{-2} & \sci{1.88}{-2} \\
Oracle per-task & \sci{1.83}{-2} & \sci{1.85}{-2} & \sci{1.81}{-2} & \sci{1.88}{-2} \\
\bottomrule
\end{tabular}

\end{table}

\FloatBarrier
\subsection{NS2D viscosity policy comparison}
\label{sec:app:ns2d-policy}

NS2D provides a viscosity-controlled incompressible-flow setting.
The local family coordinate is viscosity, which controls dissipation and long-horizon vorticity dynamics.
CCM-Scale gives the best overall validation mean among structured selectors, while CCM-Coord and CCM-Scale match on the OOD mean.
The main text reports the OOD comparison in Tab.~\ref{tab:ood-family-axis} and shows the coordinate relation and family-wise error trends in Figs.~\ref{fig:coordinate-loss-strip} and \ref{fig:cross-domain-evidence-atlas}. The full task-level and diagnostic tables are reported below in Tabs.~\ref{tab:app-method-ablation}, \ref{tab:app-task-comparison}, and \ref{tab:app-ns2d-physics}.

\FloatBarrier
\subsection{NS2D prefix-calibration budget}
\label{sec:app:ns2d-budget}

Metadata-based coordinate policies are the primary NS2D selectors.
The prefix-calibration budget shows that task-level prefix selection improves over the best fixed $\alpha$, with the lowest error at four prefix steps.
This indicates that early rollout error also carries information about the signed family direction.

\begin{figure}[!htbp]
\centering
\includegraphics[width=\linewidth,height=0.50\textheight,keepaspectratio]{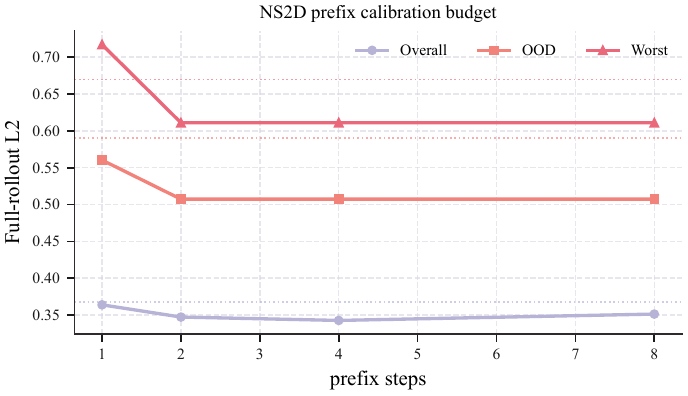}
\caption{NS2D prefix-calibration budget.}
\label{fig:app-ns2d-budget}
\end{figure}

\FloatBarrier
\subsection{RDB fixed alpha versus conditional composition}
\label{sec:app:rdb-fixed-conditional}

RDB is the least directly parameterized domain because the free-surface transient is not well described by a monotone scalar metadata coordinate.
The main text therefore uses CCM-Prefix for this setting.
The selected $\alpha$ varies by test sample and task, using the short prefix to choose the correct side of the endpoint direction instead of a single fixed coefficient.

\begin{figure}[!htbp]
\centering
\includegraphics[width=\linewidth,height=0.38\textheight,keepaspectratio]{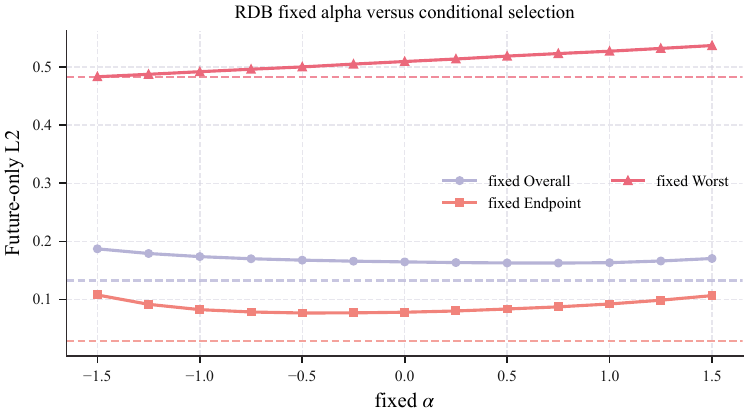}
\caption{RDB fixed-$\alpha$ frontier versus conditional CCM-Prefix.}
\label{fig:app-rdb-fixed-conditional}
\end{figure}

\begin{figure}[!htbp]
\centering
\includegraphics[width=\linewidth,height=0.38\textheight,keepaspectratio]{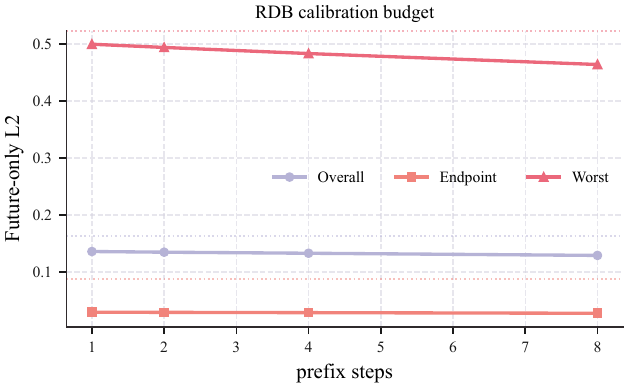}
\caption{RDB calibration-budget curve.}
\label{fig:app-rdb-budget}
\end{figure}

\begin{figure}[!htbp]
\centering
\includegraphics[width=\linewidth,height=0.38\textheight,keepaspectratio]{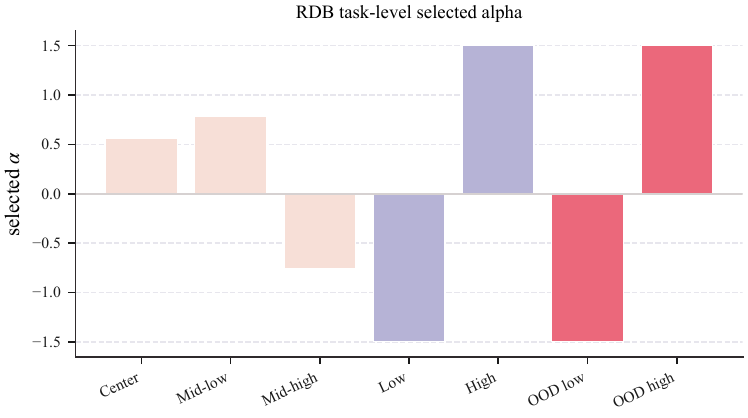}
\caption{RDB task-level selected $\alpha$.}
\label{fig:app-rdb-task-alpha}
\end{figure}

\begin{table}[!htbp]
\centering
\appdensefont
\caption{RDB selected $\alpha$ statistics by task group.
Endpoint and OOD tasks select large-magnitude coefficients, while support tasks show smaller and more variable selections. RDB is therefore treated as a prefix-calibrated regime with sample-specific coefficient selection.}
\label{tab:app-rdb-selected-alpha-groups}
\begin{tabular}{lccccc}
\toprule
\textbf{Task group} & \textbf{\# tasks} & \textbf{Mean $\alpha$} & \textbf{Std. $\alpha$} & \textbf{Mean $|\alpha|$} & \textbf{Future L2} \\
\midrule
Support & 3 & +0.20 & 0.83 & 0.70 & \sci{1.99}{-2} \\
Endpoint & 2 & +0.00 & 2.12 & 1.50 & \sci{2.85}{-2} \\
OOD & 2 & +0.00 & 2.12 & 1.50 & \sci{4.06}{-1} \\
All & 7 & +0.08 & 1.32 & 1.16 & \sci{1.33}{-1} \\
\bottomrule
\end{tabular}

\end{table}

\FloatBarrier
\subsection{Unified method ablation}
\label{sec:app:method-ablation}

The unified ablation compares static averaging, best fixed $\alpha$, wrong fixed $\alpha$, and conditional CCM on NS2D and RDB.
It measures whether conditioning changes the selected point on the merge line beyond what one fixed interpolation coefficient can provide.
Conditioning matters most when one fixed $\alpha$ cannot serve all regimes.

\begin{table}[!htbp]
\centering
\appdensefont
\caption{Unified method ablation on RDB and NS2D.}
\label{tab:app-method-ablation}
\begin{tabular}{llrrrr}
\toprule
\textbf{Benchmark} & \textbf{Method} & \textbf{Overall} & \textbf{Endpoint} & \textbf{OOD/ID} & \textbf{Worst} \\
\midrule
RDB high-center & Base & \sci{1.77}{-1} & \sci{1.22}{-1} & \sci{5.9}{-2} & \sci{5.34}{-1} \\
RDB high-center & Endpoint average & \sci{1.65}{-1} & \sci{7.8}{-2} & \sci{5.1}{-2} & \sci{5.09}{-1} \\
RDB high-center & Best fixed overall & \sci{1.63}{-1} & \sci{8.7}{-2} & \sci{5.1}{-2} & \sci{5.23}{-1} \\
RDB high-center & Best fixed endpoint & \sci{1.68}{-1} & \sci{7.7}{-2} & \sci{5.2}{-2} & \sci{5.00}{-1} \\
RDB high-center & Worst fixed & \sci{1.87}{-1} & \sci{1.08}{-1} & \sci{7.2}{-2} & \sci{4.83}{-1} \\
RDB high-center & CCM-Prefix & \sci{1.33}{-1} & \sci{2.8}{-2} & \sci{2.3}{-2} & \sci{4.83}{-1} \\
NS2D viscosity & Base & \sci{3.78}{-1} & \sci{3.10}{-1} & -- & \sci{6.78}{-1} \\
NS2D viscosity & Endpoint average & \sci{3.71}{-1} & \sci{3.05}{-1} & \sci{5.86}{-1} & \sci{7.17}{-1} \\
NS2D viscosity & Best fixed overall & \sci{3.68}{-1} & \sci{2.97}{-1} & \sci{5.90}{-1} & \sci{6.70}{-1} \\
NS2D viscosity & CCM-Coord & \sci{2.76}{-1} & \sci{2.65}{-1} & \sci{3.38}{-1} & \sci{3.47}{-1} \\
NS2D viscosity & CCM-Scale & \sci{2.53}{-1} & \sci{2.65}{-1} & \sci{3.38}{-1} & \sci{3.47}{-1} \\
NS2D viscosity & Worst fixed & \sci{5.42}{-1} & \sci{5.60}{-1} & \sci{5.44}{-1} & \sci{7.59}{-1} \\
\bottomrule
\end{tabular}

\end{table}

\FloatBarrier
\subsection{Task-level comparison against the base}
\label{sec:app:task-comparison}

Task-level comparison is reported separately from mean performance.
This separates average improvement from pointwise behavior relative to the base.
The table shows that average merge improves several merge-friendly DiffReact/FNO
settings, while harder domains include tasks where the base remains stronger.
CCM improves the mean and reduces regret, with pointwise behavior reported
explicitly at the task level.

\begin{table}[!htbp]
\centering
\appdensefont
\caption{Task-level comparison against the family base.}
\label{tab:app-task-comparison}
\begin{tabular}{llrrrr}
\toprule
\textbf{Experiment} & \textbf{Method} & \textbf{Win/Loss} & \textbf{Base} & \textbf{Method} & \textbf{Neg. regret} \\
\midrule
DiffReact f2 small-gap FNO & Endpoint average & 5/0 & \sci{1.27}{-1} & \sci{1.19}{-1} & 0 \\
DiffReact f2 dense FNO & Endpoint average & 13/0 & \sci{1.40}{-1} & \sci{1.32}{-1} & 0 \\
DiffReact f3 FNO small\_w48\_mo12\_l4 & Endpoint average & 5/0 & \sci{2.12}{-1} & \sci{1.92}{-1} & 0 \\
DiffReact f3 FNO base\_w64\_mo16\_l4 & Endpoint average & 5/0 & \sci{1.87}{-1} & \sci{1.49}{-1} & 0 \\
DiffReact f3 FNO large\_w80\_mo20\_l4 & Endpoint average & 5/0 & \sci{1.50}{-1} & \sci{1.13}{-1} & 0 \\
DiffReact f3 DPOT & Endpoint average & 3/2 & \sci{4.5}{-2} & \sci{3.9}{-2} & \sci{2}{-3} \\
RDB high-center future-only & Endpoint average & 5/2 & \sci{1.77}{-1} & \sci{1.65}{-1} & \sci{7}{-3} \\
RDB high-center future-only & CCM-Prefix & 5/2 & \sci{1.77}{-1} & \sci{1.33}{-1} & \sci{1}{-3} \\
RDB high-center future-only & CCM-Prefix+fallback & 5/2 & \sci{1.77}{-1} & \sci{1.32}{-1} & \sci{1}{-3} \\
NS2D viscosity & Endpoint average & 6/5 & \sci{3.78}{-1} & \sci{3.71}{-1} & \sci{2.1}{-2} \\
NS2D viscosity & CCM-Coord & 11/0 & \sci{3.78}{-1} & \sci{2.76}{-1} & 0 \\
NS2D viscosity & CCM-Scale & 11/0 & \sci{3.78}{-1} & \sci{2.53}{-1} & 0 \\
\bottomrule
\end{tabular}

\end{table}

\FloatBarrier
\subsection{RDB output-space ensemble reference}
\label{sec:app:output-ensemble}

Output-space ensembling is a multi-checkpoint reference because it keeps both endpoint experts alive and combines predictions at every rollout step.
Its cost profile differs from weight-space CCM: it requires two resident
checkpoints and roughly two endpoint forward passes per step.
On RDB, output-space CCM-Prefix nearly matches weight-space CCM-Prefix, indicating that the weight-space method captures much of the endpoint synergy with one composed checkpoint.

\begin{table}[!htbp]
\centering
\appdensefont
\caption{RDB output-space ensemble reference.}
\label{tab:app-output-ensemble}
\begin{tabular}{lcccrrrr}
\toprule
\textbf{Method} & \textbf{Ckpts} & \textbf{Infer.} & \textbf{Calib.} & \textbf{Overall} & \textbf{Endpoint} & \textbf{ID-5} & \textbf{Worst} \\
\midrule
Weight avg merge & 1 & 1x & none & \sci{1.65}{-1} & \sci{7.8}{-2} & \sci{5.1}{-2} & \sci{5.09}{-1} \\
Weight best fixed & 1 & 1x & fixed-alpha oracle & \sci{1.63}{-1} & \sci{8.7}{-2} & \sci{5.1}{-2} & \sci{5.23}{-1} \\
Weight CCM-Prefix & 1 & 1x & 4 rollout steps & \sci{1.33}{-1} & \sci{2.8}{-2} & \sci{2.3}{-2} & \sci{4.83}{-1} \\
Output avg & 2 & 2x & none & \sci{1.64}{-1} & \sci{8.3}{-2} & \sci{5.0}{-2} & \sci{5.10}{-1} \\
Output best fixed & 2 & 2x & fixed-alpha oracle & \sci{1.63}{-1} & \sci{8.9}{-2} & \sci{5.1}{-2} & \sci{5.23}{-1} \\
Output CCM-Prefix & 2 & 2x & 4 rollout steps & \sci{1.33}{-1} & \sci{3.2}{-2} & \sci{2.4}{-2} & \sci{4.83}{-1} \\
Output oracle per-task & 2 & 2x & oracle full rollout & \sci{1.32}{-1} & \sci{2.9}{-2} & \sci{2.2}{-2} & \sci{4.83}{-1} \\
\bottomrule
\end{tabular}

\end{table}

\FloatBarrier
\subsection{RDB alpha-bank resolution}
\label{sec:app:alpha-grid}

CCM-Prefix searches a finite $\alpha$ bank, and the $\alpha$-grid ablation varies the candidate resolution.
The 5-point, 9-point, and 13-point grids produce nearly identical CCM-Prefix performance under the same future-only evaluation protocol.

\begin{table}[!htbp]
\centering
\appdensefont
\caption{RDB $\alpha$-bank resolution sweep.}
\label{tab:app-alpha-grid}
\begin{tabular}{llrrrr}
\toprule
\textbf{Grid} & \textbf{Method} & \textbf{Overall} & \textbf{Endpoint} & \textbf{ID-5} & \textbf{Worst} \\
\midrule
5 & Best fixed & \sci{1.62706}{-1} & \sci{8.7393}{-2} & \sci{5.0929}{-2} & \sci{5.23030}{-1} \\
5 & Oracle & \sci{1.32851}{-1} & \sci{2.8456}{-2} & \sci{2.3514}{-2} & \sci{4.83108}{-1} \\
5 & CCM-Prefix & \sci{1.32850}{-1} & \sci{2.8456}{-2} & \sci{2.3512}{-2} & \sci{4.83108}{-1} \\
9 & Best fixed & \sci{1.62706}{-1} & \sci{8.7393}{-2} & \sci{5.0929}{-2} & \sci{5.23030}{-1} \\
9 & Oracle & \sci{1.31959}{-1} & \sci{2.8456}{-2} & \sci{2.2265}{-2} & \sci{4.83108}{-1} \\
9 & CCM-Prefix & \sci{1.32704}{-1} & \sci{2.8456}{-2} & \sci{2.3308}{-2} & \sci{4.83108}{-1} \\
13 & Best fixed & \sci{1.62706}{-1} & \sci{8.7393}{-2} & \sci{5.0929}{-2} & \sci{5.23030}{-1} \\
13 & Oracle & \sci{1.31938}{-1} & \sci{2.8384}{-2} & \sci{2.2236}{-2} & \sci{4.83108}{-1} \\
13 & CCM-Prefix & \sci{1.32708}{-1} & \sci{2.8473}{-2} & \sci{2.3314}{-2} & \sci{4.83108}{-1} \\
\bottomrule
\end{tabular}

\end{table}

\FloatBarrier
\subsection{RDB calibration-objective ablation}
\label{sec:app:calibration-objective}

The calibration objective ablation measures sensitivity to the exact prefix loss.
First-step, mean-step, final-step, recency-weighted, and complete-prefix L2 selectors all remain far better than the best fixed $\alpha$.
In the table, ``Complete prefix L2'' denotes the selector that uses all $K$ calibration steps for coefficient selection, not a reported full-rollout score.
These variants show that short rollout information provides a stable selection signal for the signed family direction.

\begin{table}[!htbp]
\centering
\appdensefont
\caption{RDB calibration-objective ablation.}
\label{tab:app-calibration-objective}
\begin{tabular}{lrrrr}
\toprule
\textbf{Selection objective} & \textbf{Overall} & \textbf{Endpoint} & \textbf{ID-5} & \textbf{Worst} \\
\midrule
Best fixed & \sci{1.62706}{-1} & \sci{8.7393}{-2} & \sci{5.0929}{-2} & \sci{5.23030}{-1} \\
Complete prefix L2 & \sci{1.32708}{-1} & \sci{2.8473}{-2} & \sci{2.3314}{-2} & \sci{4.83108}{-1} \\
Mean step L2 & \sci{1.32680}{-1} & \sci{2.8476}{-2} & \sci{2.3274}{-2} & \sci{4.83109}{-1} \\
First-step L2 & \sci{1.32920}{-1} & \sci{2.8458}{-2} & \sci{2.3610}{-2} & \sci{4.83109}{-1} \\
Final-step L2 & \sci{1.32602}{-1} & \sci{2.8491}{-2} & \sci{2.3165}{-2} & \sci{4.83109}{-1} \\
Recency-weighted L2 & \sci{1.32635}{-1} & \sci{2.8491}{-2} & \sci{2.3211}{-2} & \sci{4.83109}{-1} \\
\bottomrule
\end{tabular}

\end{table}

\FloatBarrier
\subsection{RDB bootstrap confidence intervals}
\label{sec:app-bootstrap}

The bootstrap intervals below are sample-level intervals under the same future-only evaluation protocol.
They quantify resampling variation for the measured RDB conditional gain, while the matched independent-seed reruns in Appendix~\ref{sec:app:seed-reruns} report targeted independent-seed replication.

\begin{table}[!htbp]
\centering
\appdensefont
\caption{RDB sample-level bootstrap confidence intervals.}
\label{tab:app-bootstrap}
\begin{tabular}{llrr}
\toprule
\textbf{Baseline $-$ CCM-Prefix} & \textbf{Metric} & \textbf{Improvement} & \textbf{95\% CI} \\
\midrule
Base & Overall & \sci{4.4105}{-2} & [\sci{4.3969}{-2}, \sci{4.4244}{-2}] \\
Best fixed overall & Overall & \sci{2.9998}{-2} & [\sci{2.9922}{-2}, \sci{3.0079}{-2}] \\
Best fixed endpoint & Endpoint & \sci{4.8356}{-2} & [\sci{4.7953}{-2}, \sci{4.8761}{-2}] \\
Best fixed ID & ID-5 & \sci{2.6880}{-2} & [\sci{2.6785}{-2}, \sci{2.6975}{-2}] \\
Best fixed overall & Worst & \sci{3.9922}{-2} & [\sci{3.9753}{-2}, \sci{4.0091}{-2}] \\
\bottomrule
\end{tabular}

\end{table}

\FloatBarrier
\subsection{Matched independent-seed reruns}
\label{sec:app:seed-reruns}

Matched three-seed reruns complement the sample-level bootstrap intervals in two targeted settings.
DiffReact dense-axis OOD evaluates the coordinate-law mechanism under independent endpoint reruns.
RDB future-only CCM-Prefix evaluates the prefix-calibration mechanism under independent endpoint-seed reruns.
These settings cover the two coefficient-selection mechanisms used in the paper: metadata coordinate selection and short-prefix calibration.
The table reports mean and standard deviation under the matched rerun protocol.

\begin{table}[!htbp]
\centering
\appdensefont
\caption{Matched independent expert-seed reruns for coordinate selection and prefix calibration.
For DiffReact, the primary metric is OOD full-rollout L2 on the dense physical axis; the CCM-Coord row uses the known coordinate $\alpha=s(\lambda)$, while the oracle row uses the per-task best $\alpha$ from the bank and matches CCM-Coord on this dense sweep.
For RDB, the primary metric is prefix-excluded future-rollout L2 evaluated after the $K=4$ observed prefix used by CCM-Prefix; the Overall L2 column remains the corresponding full family average reported by the rerun script.}
\label{tab:app-score8-seed-reruns}
\begin{tabular}{lllcc}
\toprule
\textbf{Domain} & \textbf{Policy} & \textbf{Selection} & \textbf{Primary metric} & \textbf{Overall L2} \\
\midrule
DiffReact dense & Base & no merge & \sci{1.88}{-1} $\pm$ \sci{0}{0} & \sci{1.4}{-1} \\
 & Endpoint average & fixed zero & \sci{1.84}{-1} $\pm$ \sci{3.98}{-4} & \sci{1.32}{-1} \\
 & Best fixed OOD & varies & \sci{1.83}{-1} $\pm$ \sci{3.52}{-4} & \sci{1.35}{-1} \\
 & CCM-Coord & $\alpha=s(\lambda)$ & \sci{8.64}{-2} $\pm$ \sci{1.75}{-4} & \sci{8.25}{-2} \\
 & Oracle bank & per-task best & \sci{8.64}{-2} $\pm$ \sci{1.75}{-4} & \sci{8.25}{-2} \\
RDB future-only & Best fixed overall & $\alpha=0.75$ & \sci{9.38}{-2} $\pm$ \sci{3.75}{-4} & \sci{1.68}{-1} \\
 & Best fixed OOD & varies & \sci{7.87}{-2} $\pm$ \sci{7.08}{-4} & \sci{1.74}{-1} \\
 & CCM-Prefix & $K=4$ & \sci{2.89}{-2} $\pm$ \sci{1.9}{-4} & \sci{1.33}{-1} \\
\bottomrule
\end{tabular}

\end{table}

\FloatBarrier
\subsection{Physics-aware metrics}
\label{sec:app:physics-compute}

Physics-aware metrics provide additional measurements of the rollout results.
Rollout L2 remains the primary training and evaluation target, and each table
preserves the full diagnostic profile for the compared policies.

\begin{table}[!htbp]
\centering
\appdensefont
\caption{Radial dam-break physics-aware metrics.
Future L2 denotes the prefix-excluded rollout L2 used by the RDB CCM-Prefix protocol; the remaining columns are additional rollout diagnostics.}
\label{tab:app-rdb-physics}
\begin{tabular}{lrrrr}
\toprule
\textbf{Method} & \textbf{Future L2 $\downarrow$} & \textbf{Mass MAE $\downarrow$} & \textbf{Std MAE $\downarrow$} & \textbf{Front MAE $\downarrow$} \\
\midrule
Base & \sci{1.768}{-1} & \sci{6.8}{-3} & \sci{7.4}{-3} & \sci{1.49}{-2} \\
Best fixed overall & \sci{1.627}{-1} & \sci{6.5}{-3} & \sci{7.1}{-3} & \sci{1.28}{-2} \\
Best fixed endpoint & \sci{1.677}{-1} & \sci{8.5}{-3} & \sci{8.6}{-3} & \sci{1.35}{-2} \\
Fixed ID & \sci{1.629}{-1} & \sci{6.7}{-3} & \sci{7.2}{-3} & \sci{1.30}{-2} \\
CCM-Prefix & \sci{1.327}{-1} & \sci{5.2}{-3} & \sci{3.8}{-3} & \sci{1.15}{-2} \\
\bottomrule
\end{tabular}

\end{table}

\FloatBarrier
\begin{table}[!htbp]
\centering
\appdensefont
\caption{NS2D physics-aware metrics.
Full L2 denotes full-rollout L2; the remaining columns are additional rollout diagnostics.}
\label{tab:app-ns2d-physics}
\begin{tabular}{lrrrr}
\toprule
\textbf{Method} & \textbf{Full L2 $\downarrow$} & \textbf{Vort. mean MAE $\downarrow$} & \textbf{Enstrophy MAE $\downarrow$} & \textbf{Final Enstrophy $\downarrow$} \\
\midrule
Base & \sci{3.780}{-1} & \sci{3.0}{-3} & \sci{3.172}{0} & \sci{5.896}{0} \\
Endpoint average & \sci{3.705}{-1} & \sci{3.0}{-3} & \sci{3.165}{0} & \sci{5.861}{0} \\
Best fixed overall & \sci{3.679}{-1} & \sci{2.8}{-3} & \sci{3.172}{0} & \sci{5.863}{0} \\
CCM-Coord & \sci{2.762}{-1} & \sci{3.0}{-3} & \sci{3.182}{0} & \sci{4.356}{0} \\
\bottomrule
\end{tabular}

\end{table}

\FloatBarrier
\section{Qualitative Case Studies}
\label{sec:app:case-studies}

Late-time qualitative panels show medium-gap family tasks.
These panels complement scalar error tables by showing how the controlled physical axis changes the visible PDE state after autoregressive errors have had time to accumulate.
Each panel uses a fixed seed and later time slices cropped from the full same-seed task strips.

Figure~\ref{fig:app-diffreact-late-medium-gap} shows the DiffReact $f_2$ medium-gap family.
The rows vary the $D_u$ coefficient while keeping the same seed and late-time frames.
Lower $D_u$ retains sharper spatial variation in the $u$ channel, whereas higher $D_u$ produces smoother late-time concentration fields.
This is the field-level counterpart of the quantitative $f_2$ result: the medium-gap axis is a visible physical coordinate that changes the long-horizon solution morphology.

\begin{figure}[!htbp]
\centering
\includegraphics[width=0.92\linewidth,height=0.46\textheight,keepaspectratio]{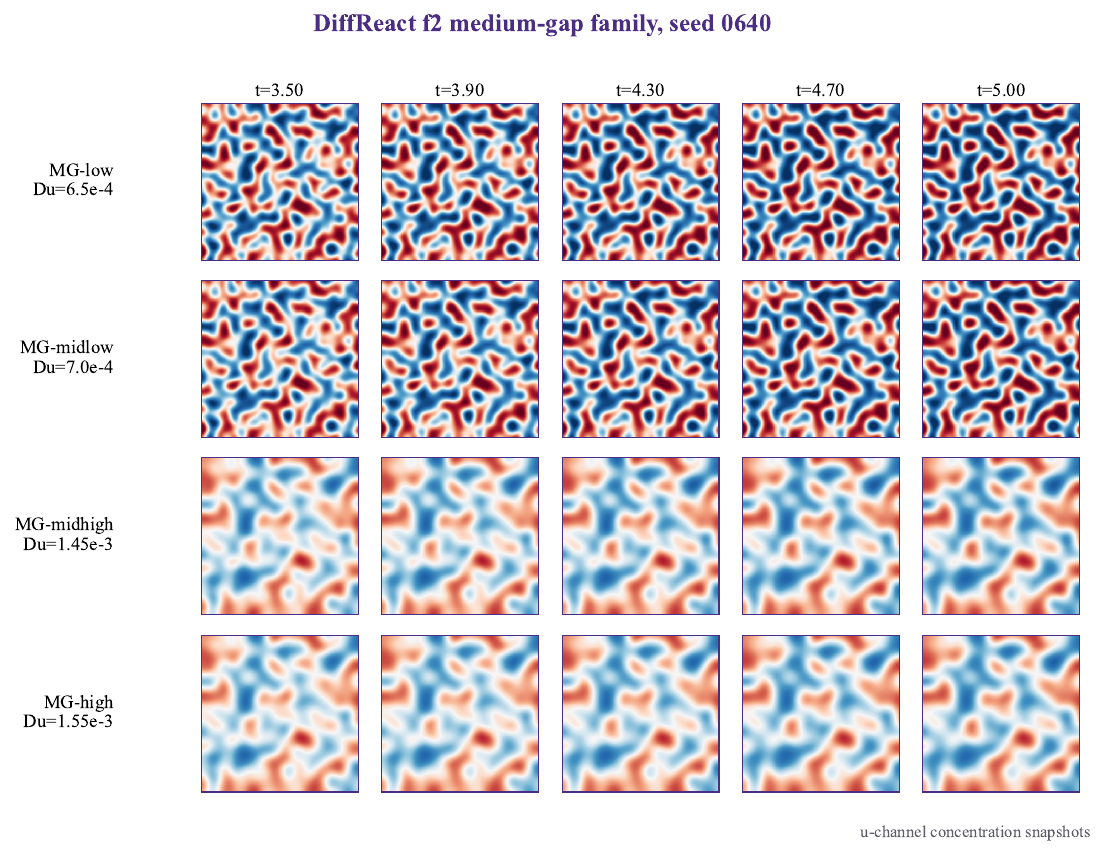}
\caption{\textbf{Late-time DiffReact $f_2$ medium-gap visualization.}
Rows correspond to four medium-gap $D_u$ tasks for the same seed.
Columns show late rollout frames from $t=3.50$ to $t=5.00$; the $u$ channel is
shown for compactness.}
\label{fig:app-diffreact-late-medium-gap}
\end{figure}

Figure~\ref{fig:app-ns2d-late-medium-gap} shows the NS2D viscosity family.
The two rows are the medium-gap low and high viscosity tasks.
At late times, the lower-viscosity trajectory retains sharper rolled-up vorticity filaments, whereas the higher-viscosity trajectory becomes visibly smoother.
This is the same physical axis used by the NS2D $\alpha$ law and CCM-Scale experiments; the qualitative panel therefore verifies that the medium-gap coordinate corresponds to a clear dynamical change in the field as well as a numerical label.

\begin{figure}[!htbp]
\centering
\includegraphics[width=0.92\linewidth,height=0.36\textheight,keepaspectratio]{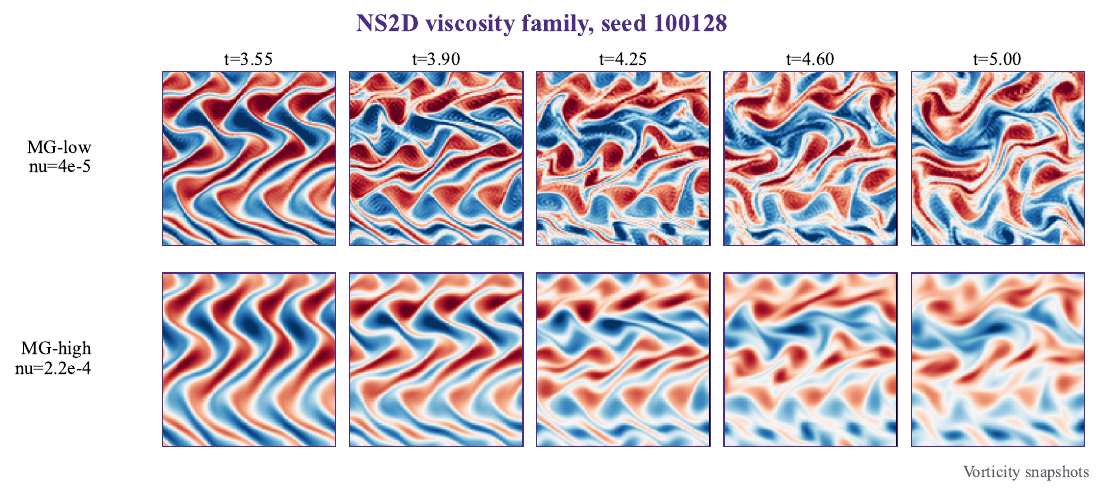}
\caption{\textbf{Late-time NS2D medium-gap visualization.}
Rows correspond to medium-gap low/high viscosity tasks for the same seed.
Columns show late rollout frames from $t=3.55$ to $t=5.00$.}
\label{fig:app-ns2d-late-medium-gap}
\end{figure}

Fig.~\ref{fig:app-rdb-late-medium-gap} shows the radial dam-break
high-center family.
Here the medium-gap coordinate changes the initial central height.
Late frames reveal the downstream effect on the expanding free-surface ring: larger heights produce stronger outward propagation and a larger low-height central basin.
This panel illustrates why RDB is a prefix-calibrated setting: the visible state changes are strong, while the signed family direction is less reliably determined by static metadata alone.

\begin{figure}[!htbp]
\centering
\includegraphics[width=0.92\linewidth,height=0.46\textheight,keepaspectratio]{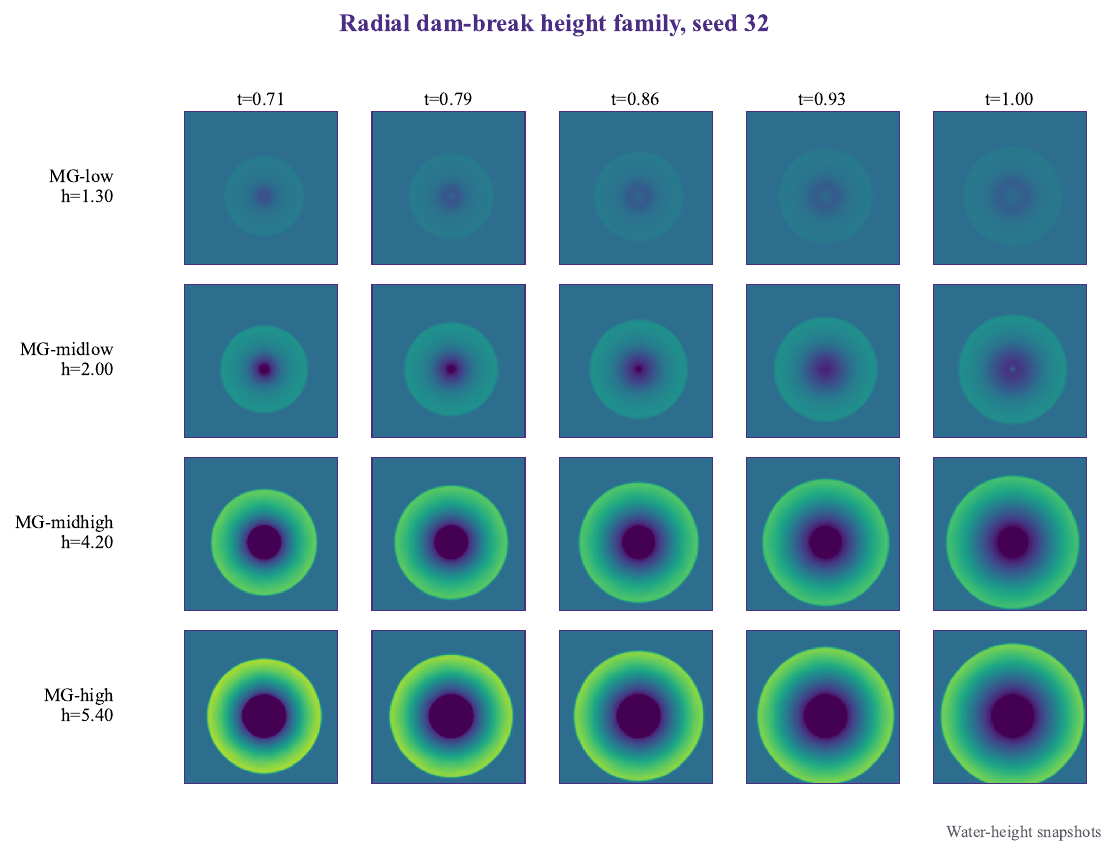}
\caption{\textbf{Late-time RDB medium-gap visualization.}
Rows correspond to four medium-gap height tasks for the same seed.
Columns show late rollout frames from $t=0.71$ to $t=1.00$.}
\label{fig:app-rdb-late-medium-gap}
\end{figure}

\end{document}